\documentclass{article}

\PassOptionsToPackage{numbers, compress}{natbib}

\usepackage[final]{neurips_2022}

\usepackage[utf8]{inputenc} 
\usepackage[T1]{fontenc}    
\usepackage{hyperref}       
\usepackage{url}            
\usepackage{booktabs}       
\usepackage{amsfonts}       
\usepackage{nicefrac}       
\usepackage{microtype}      
\usepackage{xcolor}         

\newcommand{\lin}{\text{lin}}
\newcommand{\gradnorm}{\text{Grad}}
\newcommand{\tracenorm}{\text{Trace}} 
\newcommand{\snip}{\text{SNIP}}
\newcommand{\grasp}{\text{GraSP}}
\newcommand{\synflow}{\text{SynFlow}}

\newcommand{\fisher}{\text{Fisher}}
\newcommand{\naswot}{\text{NASWOT}}
\newcommand{\knas}{\text{KNAS}}

\newcommand{\frob}{\text{F}}
\newcommand{\tr}{\text{tr}}

\newcommand{\ms}{\phantom{1}}

\usepackage[title]{appendix}
\usepackage{bm}
\usepackage{multicol}
\usepackage{amsmath}
\usepackage{amssymb}
\usepackage{mathtools}
\usepackage{array}
\usepackage{multirow}
\usepackage{float}
\usepackage{amsthm}
\usepackage{xr}
\usepackage{thmtools, thm-restate}
\usepackage{threeparttable}
\usepackage{nicefrac}
\usepackage{enumerate}
\usepackage{siunitx}
\usepackage{algorithmic}
\usepackage[ruled,vlined]{algorithm2e}
\usepackage{subcaption}
\usepackage{caption}

\newcolumntype{L}[1]{>{\raggedright\let\newline\\\arraybackslash\hspace{0pt}}m{#1}}
\newcolumntype{C}[1]{>{\centering\let\newline\\\arraybackslash\hspace{0pt}}m{#1}}
\newcolumntype{R}[1]{>{\raggedleft\let\newline\\\arraybackslash\hspace{0pt}}m{#1}}

\usepackage{amsmath,amsfonts,bm}

\def\1{\bm{1}}

\def\eps{{\epsilon}}

\def\rx{{\textnormal{x}}}
\def\ry{{\textnormal{y}}}

\def\rmA{{\mathbf{A}}}

\def\rmH{{\mathbf{H}}}
\def\rmI{{\mathbf{I}}}

\def\rmV{{\mathbf{V}}}
\def\rmW{{\mathbf{W}}}
\def\rmX{{\mathbf{X}}}

\def\rmTheta{{\mathbf{\Theta}}}
\def\rmLambda{{\mathbf{\Lambda}}}

\def\vzero{{\bm{0}}}
\def\vone{{\bm{1}}}

\def\vtheta{{\bm{\theta}}}
\def\vTheta{{\bm{\Theta}}}

\def\vepsilon{{\bm{\epsilon}}}

\def\vv{{\bm{v}}}
\def\vw{{\bm{w}}}
\def\vx{{\bm{x}}}
\def\vy{{\bm{y}}}

\DeclareMathAlphabet{\mathsfit}{\encodingdefault}{\sfdefault}{m}{sl}
\SetMathAlphabet{\mathsfit}{bold}{\encodingdefault}{\sfdefault}{bx}{n}

\def\gA{{\mathcal{A}}}
\def\gB{{\mathcal{B}}}

\def\gD{{\mathcal{D}}}

\def\gF{{\mathcal{F}}}
\def\gG{{\mathcal{G}}}
\def\gH{{\mathcal{H}}}

\def\gL{{\mathcal{L}}}
\def\gM{{\mathcal{M}}}

\def\gO{{\mathcal{O}}}
\def\gP{{\mathcal{P}}}

\def\gR{{\mathcal{R}}}

\def\sN{{\mathbb{N}}}

\def\sP{{\mathbb{P}}}

\def\sR{{\mathbb{R}}}

\newcommand{\E}{\mathbb{E}}
\newcommand{\Ls}{\mathcal{L}}

\DeclareMathOperator*{\argmax}{arg\,max}
\DeclareMathOperator*{\argmin}{arg\,min}

\newtheorem{theorem}{Theorem}
\newtheorem{corollary}{Corollary}

\newtheorem{lemma}{Lemma}
\newtheorem*{remark}{Remark}

\usepackage{apptools}
\AtAppendix{\counterwithin{lemma}{section}}
\AtAppendix{\counterwithin{theorem}{section}}
\AtAppendix{\counterwithin{definition}{section}}

\let\svthefootnote\thefootnote
\newcommand\blankfootnote[1]{%
  \let\thefootnote\relax\footnotetext{#1}%
  \let\thefootnote\svthefootnote%
}

\title{Unifying and Boosting Gradient-Based Training-Free Neural Architecture Search}

\author{%
Yao Shu$^{\dagger}$, Zhongxiang Dai$^{\dagger}$, Zhaoxuan Wu$^{\S}$, Bryan Kian Hsiang Low$^{\dagger}$\\
Dept. of Computer Science, National University of Singapore, Republic of Singapore$^{\dagger}$ \\
Institute of Data Science, National University of Singapore, Republic of Singapore$^\S$ \\
Integrative Sciences and Engineering Programme, NUSGS, Republic of Singapore$^\S$ \\
\texttt{\{shuyao,daizhongxiang,lowkh\}@comp.nus.edu.sg}$^{\dagger}$ \\
\texttt{wu.zhaoxuan@u.nus.edu}$^{\S}$
}

\begin{document}

\maketitle

\begin{abstract}
\emph{Neural architecture search} (NAS) has gained immense popularity owing to its ability to automate neural architecture design.
A number of training-free metrics are recently proposed to realize NAS without training, hence making NAS more scalable.
Despite their competitive empirical performances, a unified theoretical understanding of these training-free metrics is lacking.
As a consequence, \emph{(a)} the relationships among these metrics are unclear, \emph{(b)} there is no theoretical interpretation for their empirical performances, and \emph{(c)} there may exist untapped potential in existing training-free NAS, which probably can be unveiled through a unified theoretical understanding.
To this end, this paper presents a unified theoretical analysis of gradient-based training-free NAS, which allows us to
\emph{(a)} theoretically study their relationships, \emph{(b)} theoretically guarantee their generalization performances, and \emph{(c)} exploit our unified theoretical understanding to develop a novel framework named \emph{hybrid NAS} (HNAS) which consistently boosts training-free NAS in a principled way.
Remarkably, HNAS can enjoy the advantages of both training-free (i.e., the superior search efficiency) and training-based (i.e., the remarkable search effectiveness) NAS, which we have demonstrated through extensive experiments.
\end{abstract}

\section{Introduction}
\blankfootnote{Correspondence to: Zhongxiang Dai <daizhongxiang@comp.nus.edu.sg>.}
Recent years have witnessed a surging interest in applying \emph{deep neural networks} (DNNs) in real-world applications, e.g., machine translation \citep{machine-translation-survey}, object detection \citep{object-detection-survey}, among others. To achieve compelling performances in these applications, many domain-specific neural architectures have been handcrafted by human experts with considerable efforts. However, these efforts have gradually become unaffordable due to the growing demand for customizing neural architectures for different tasks. To this end, \emph{neural architecture search} (NAS) \citep{nas} has been proposed to design neural architectures automatically. While many \emph{training-based} NAS algorithms \citep{enas, darts} have achieved state-of-the-art (SOTA) performances in various tasks, their search costs usually are unaffordable in resource-constrained scenarios mainly due to their requirement for training DNNs during search. As a result, a number of training-free metrics have been developed to realize \emph{training-free NAS} \citep{naswot, te-nas}.
Surprisingly, these training-free NAS algorithms are able to achieve competitive empirical performances even compared with other training-based NAS algorithms while incurring significantly reduced search costs. Moreover, the architectures selected by these training-free NAS algorithms have been empirically found to transfer well to different tasks \citep{te-nas, nasi}.

Despite the impressive empirical performances of the NAS algorithms using training-free metrics, \emph{a unified theoretical analysis of these training-free metrics} is still lacking in the literature, leading to a few significant implications.
Firstly, the theoretical relationships of these training-free metrics are unclear, making it challenging to explain \emph{why they usually lead to comparable empirical results} \citep{zero-cost}.
Secondly, there is no theoretical guarantee for the empirically observed compelling performances of the architectures selected by NAS algorithms using these training-free metrics. As a consequence, the reason \emph{why NAS using these training-free metrics works well} is still not well understood, and hence there lacks theoretical assurances for NAS practitioners when deploying these algorithms.
To the best of our knowledge, the theoretical aspect of NAS with training-free metrics has only been preliminarily studied by \citet{nasi}. However, their analyses are only based on the training rather than generalization performances of different architectures and are restricted to a \emph{single} training-free metric.
Thirdly, there may exist untapped potential in existing training-free NAS algorithms, which probably can be unveiled through a unified theoretical understanding of their training-free metrics.

To this end, we perform a unified theoretical analysis of \emph{gradient-based} training-free NAS to resolve all the three problems discussed above in this paper.
Firstly, we theoretically prove the connections among different gradient-based training-free metrics in Sec.~\ref{sec:connection}.
Secondly, based on these provable connections, we derive a unified generalization bound for DNNs with these metrics and then use it to provide principled interpretations for the compelling empirical performances of existing training-free NAS algorithms (Secs.~\ref{sec:generalization} and \ref{sec:generalization-nas}).
Moreover, we demonstrate that our theoretical interpretation for training-free NAS algorithms, surprisingly, displays the same preference of architecture topology (i.e., wide or deep) as training-based NAS algorithms under certain conditions (Sec.~\ref{sec:topology}), which helps to justify the practicality of our theoretical interpretations.
Thirdly, by exploiting our unified theoretical analysis, we develop a novel NAS framework named \emph{hybrid NAS} (HNAS) to consistently boost existing training-free NAS algorithms (Sec.~\ref{sec:hnas}) in a principled way.
Remarkably, through a theory-inspired combination with \emph{Bayesian optimization} (BO), our HNAS framework enjoys the advantages of both training-based (i.e., remarkable search effectiveness) and training-free (i.e., superior search efficiency) NAS simultaneously, making it more advanced than existing training-free and training-based NAS algorithms.
Lastly, we use extensive experiments to verify the insights derived from our unified theoretical analysis, as well as the search effectiveness and efficiency of our non-trivial HNAS framework (Sec.~\ref{sec:exps}).

\section{Related Works}

Recently, a number of training-free metrics have been proposed to estimate the generalization~performances of neural architectures, allowing the model training in NAS to be completely avoided.
For instance, \citet{naswot} have developed a heuristic metric using the correlation of activations in an initialized DNN. Meanwhile, \citet{zero-cost} have empirically revealed a large correlation between training-free metrics that were formerly applied in network pruning, e.g., SNIP \citep{snip} and GraSP \citep{grasp}, and the generalization performances of candidate architectures in the search space. 
These results hence indicate the feasibility of using training-free metrics to estimate the performances of candidate architectures in NAS. 
\citet{te-nas} have proposed a heuristic metric to trade off the trainability and expressibility of neural architectures in order to
find well-performing architectures in various NAS benchmarks.
\citet{knas} have applied the mean of the Gram matrix of gradients to quickly estimate the performances of architectures. More recently, \citet{nasi} have employed the theory of \emph{Neural Tangent Kernel} (NTK) \citep{ntk} to formally derive a performance estimator using the trace norm of the NTK matrix with initialized model parameters, which, surprisingly, is shown to be data- and label-agnostic. Though these existing works have demonstrated the feasibility of training-free NAS through their compelling empirical results, the reason as to \emph{why training-free NAS performs well in practice} and the answer to the question of \emph{how training-free NAS can be further boosted} remain mysteries in the literature. This paper aims to provide theoretically grounded answers to these two questions through a unified analysis of existing gradient-based training-free metrics.

\section{Notations and Backgrounds}\label{sec:notations}
\subsection{Neural Tangent Kernel}\label{sec:ntk}
To simplify the analysis in this paper, we consider a $L$-layer DNN with identical widths $n_1=\cdots=n_{L-1}=n$ and scalar output (i.e., $n_L=1$) based on the formulation of DNNs in \citep{ntk}. Let $f(\vx, \vtheta)$ be the output of a DNN with input $\vx \in \sR^{n_0}$ and parameters $\vtheta \in \sR^{d}$ that are initialized using the standard normal distribution, the NTK matrix $\rmTheta \in \sR^{m \times m}$ over a dataset of size $m$ is defined as
\begin{equation}
    \rmTheta(\vx, \vx'; \vtheta) = \nabla_{\vtheta}f(\vx,\vtheta)^{\top}\nabla_{\vtheta}f(\vx',\vtheta) \ .
\end{equation}

\citet{ntk} have shown that this NTK matrix $\rmTheta$ will finally converge to a deterministic form $\rmTheta_{\infty}$ in the infinitely wide DNN model. Meanwhile, \citet{exact-ntk, over-parameterization} have further proven that a similar result, i.e., $\rmTheta \approx \rmTheta_{\infty}$, can also be achieved in over-parameterized DNNs of finite width. Besides, \citet{exact-ntk, ntk-linear} have revealed that the training dynamics of DNNs can be well-characterized using this NTK matrix at initialization (i.e., $\rmTheta_0$ based on the initialized model parameters $\vtheta_0$) under certain conditions. More recently, \citet{tensor-program-2b} have further demonstrated that these conclusions about NTK matrix shall also hold for DNNs with any reasonable architecture, even including recurrent neural networks (RNNs) and graph neural networks (GNNs). Therefore, the conclusions drawn based on the formulation above in this paper are expected to be applicable to the NAS search spaces with complex architectures, which we will validate empirically.

\subsection{Gradient-Based Training-Free Metrics for NAS}\label{sec:metrics}
In this paper, we mainly focus on the study of those \emph{gradient-based} training-free metrics, i.e., the training-free metrics that are derived from the gradients of initialized model parameters, which we introduce below. Previous works have empirically shown that better model performances are usually associated with larger values of these training-free metrics in practice~\citep{zero-cost}.

\paragraph{Gradient norm of initialized model parameters.} 
While \citet{zero-cost} were the first to employ the gradient norm of initialized model parameters to estimate the generalization performance of candidate architectures, the same form has also been derived by \citet{nasi} to approximate their training-free metric efficiently. Following the notations in Sec.~\ref{sec:ntk}, let $\ell(\cdot, \cdot)$ be the loss function, we define the gradient norm over dataset $S=\{(\vx_i, y_i)\}_{i=1}^m$ as
\begin{equation}\label{eq:metric:gradnorm}
    \gM_{\gradnorm} \triangleq \left\|\frac{1}{m}\sum_{i=1}^m\nabla_{\vtheta} \ell(f(\vx_i, \vtheta_0), y_i) \right\|_2 \ .
\end{equation}

\paragraph{\snip\  and \grasp.}
$\snip$~\citep{snip} and $\grasp$~\citep{grasp} were originally proposed for training-free network pruning, and \citet{zero-cost} have applied them in training-free NAS to estimate the performances of candidate architectures without model training. Following the notations in Sec.~\ref{sec:ntk}, let $\rmH_i \in \sR^{d \times d}$ denote the hessian matrix induced by input $\vx_i$, the metrics of $\snip$ and $\grasp$ on dataset $S=\{(\vx_i, y_i)\}_{i=1}^m$ can be defined as
\begin{equation}
\begin{aligned}
    \gM_{\snip} \triangleq \left|\frac{1}{m}\sum_i^m\vtheta_0^{\top} \nabla_{\vtheta}\ell(f(\vx_i, \vtheta_0), y_i)\right|\ , \; 
    \gM_{\grasp} \triangleq \left|\frac{1}{m}\sum_i^m  \vtheta_0^{\top}\left(\rmH_i\nabla_{\vtheta}\ell(f(\vx_i, \vtheta_0), y_i)\right)\right| \ .
\end{aligned}
\end{equation}
Of note, we use the scaled (i.e, by $1/m$) absolute value of the original $\grasp$ metric in \citep{grasp} throughout this paper to match the mathematical form of other training-free metrics. 

\paragraph{Trace norm of NTK matrix at initialization.}
Recently, \citet{nasi} have reformulated NAS into a constrained optimization problem to maximize the trace norm of the NTK matrix at initialization. In addition, \citet{nasi} have empirically shown that this trace norm is highly correlated with the generalization performance of candidate architectures under their derived constraint. Let $\rmTheta_0$ be the NTK matrix based on initialized model parameters $\vtheta_0$ of a DNN, then without considering the constraint in \citep{nasi}, we frame this training-free metric on dataset $S=\{(\vx_i, y_i)\}_{i=1}^m$ as
\begin{equation}\label{eq:metric:tracenorm}
    \gM_{\tracenorm} \triangleq \sqrt{\|\rmTheta_0\|_{\tr} / m} \ .
\end{equation}

\section{Theoretical Analyses of Training-Free NAS}\label{sec:analyses}
\subsection{Connections among Training-Free Metrics}\label{sec:connection}
Notably, though the gradient-based training-free metrics introduced in Sec.~\ref{sec:metrics} seem to have distinct mathematical forms, most of them will actually achieve similar empirical performances in practice~\citep{zero-cost}. More interestingly, these metrics in fact share the similarity of using the gradients of initialized model parameters in their calculations. 
Based on these facts, we propose the following hypothesis to explain the similar performances achieved by different training-free metrics in Sec.~\ref{sec:metrics}: \emph{The training-free metrics in Sec.~\ref{sec:metrics} may be theoretically connected and hence could provide similar characterization for the generalization performances of neural architectures}. We validate this hypothesis affirmatively and use the following theorem to establish the theoretical connections among these metrics.

\begin{theorem}\label{th:connection}
Let the loss function $\ell(\cdot, \cdot)$ in gradient-based training-free metrics be $\beta$-Lipschitz~continuous and $\gamma$-Lipschitz smooth in the first argument. There exist the constant $C_1, C_2, C_3 >0$ such that the following holds with a high probability, 
\begin{equation*}
\gM_{\normalfont\gradnorm} \leq C_1\gM_{\normalfont\tracenorm}, \; \gM_{\normalfont\snip} \leq C_2\gM_{\normalfont\tracenorm}, \; \gM_{\normalfont\grasp} \leq C_3\gM_{\normalfont\tracenorm}\ .
\end{equation*}
\end{theorem}

The proof of Theorem~\ref{th:connection} are given in Appendix~\ref{sec:proof:connection}. Notably, our Theorem~\ref{th:connection} implies that with a high probability, architectures of larger $\gM_{\gradnorm}$, $\gM_{\snip}$ or $\gM_{\gradnorm}$ will also achieve a larger $\gM_{\tracenorm}$ given the inequalities above. That is, the value of $\gM_{\gradnorm}$, $\gM_{\snip}$ and $\gM_{\gradnorm}$ for different architectures in the NAS search space should be highly correlated with the value of $\gM_{\tracenorm}$. As a consequence, these training-free metrics should be able to provide similar estimation of the generalization performances of architectures (validated in Sec.~\ref{sec:exp:generalization}) and hence similar performances can be achieved when using these metrics (validated in Sec.~\ref{sec:exp:hnas}).
Overall, the training-free NAS metrics from Sec.~\ref{sec:metrics} can all be theoretically connected with $\gM_{\tracenorm}$ despite their distinct mathematical forms. Note that though our Theorem~\ref{th:connection} is only able to establish the theoretical connections between $\gM_{\tracenorm}$ and other training-free metrics, our empirical results in Appendix~\ref{sec:app:connection} further reveal that \textit{any two} training-free metrics from Sec.~\ref{sec:metrics} will also be highly correlated. Interestingly, these results also serve as principled justifications for the similar performances achieved by these training-free metrics in \citep{zero-cost}.

\subsection{A Generalization Bound Induced by Training-free Metrics}\label{sec:generalization}
Let dataset $S{=}\{(\vx_i, y_i)\}_{i=1}^m$ be randomly sampled from a data distribution $\gD$, we denote $\Ls_{S}(\cdot)$ as the training error on $S$ and $\Ls_{\gD}(\cdot)$ as the corresponding \emph{generalization error} on $\gD$. Intuitively, a smaller generalization error indicates a better \emph{generalization performance}.
Thanks to the common theoretical underpinnings of gradient-based training-free metrics formalized by Theorem~\ref{th:connection}, we can perform a \emph{unified generalization analysis} for DNNs in terms of these metrics by making use of the NTK theory \citep{ntk}. 
Define $\ell(f, y) \triangleq (f-y)^2/2$ and $\eta_0 \triangleq \min\{2n^{-1}(\lambda_{\min}(\vTheta_{\infty}) + \lambda_{\max}(\vTheta_{\infty}))^{-1}, m\lambda_{\max}^{-1}(\rmTheta_0)\}$ with $\lambda_{\min}(\cdot), \lambda_{\max}(\cdot)$ being the 
minimum and maximum eigenvalue of a matrix respectively, we derive the following theorem:

\begin{theorem}\label{th:generalization}
Assume $\left\|\vx_i\right\|_2 \leq 1$ and $f(\vx_i,\vtheta_0),\lambda_{\min}(\rmTheta_0),y_i \in [0,1]$ for any $(\vx_i, y_i) \in S$. There exists a constant $N \in \sN$ such that for any $n>N$, when applying gradient descent with learning rate $\eta < \eta_0$, the generalization error of $f_t$ at time $t>0$ can be bounded as below with a high probability,
\begin{equation*}
    \Ls_{\gD}(f_t) \leq \Ls_{S}(f_t) + \gO(\kappa / \gM) \ .
\end{equation*}
Here, $\gM$ can be any metric in Sec.~\ref{sec:metrics} and $\kappa \triangleq \lambda_{\max}(\rmTheta_0)/\lambda_{\min}(\rmTheta_0)$ is the condition number of~$\rmTheta_0$.
\end{theorem}

Its proof is in Appendix~\ref{sec:proof:th:generalization} and the second term $\gO(\kappa / \gM)$ in Theorem \ref{th:generalization} represents the \emph{generalization gap} of DNN models.
Notably, our Theorem \ref{th:generalization} provides an explicit theoretical connection between the gradient-based training-free metrics from Sec.~\ref{sec:metrics} and the generalization gap of DNNs, which later serves as the foundation to theoretically interpret the compelling performances achieved by existing training-free NAS algorithms (Sec.~\ref{sec:generalization-nas}). Compared to the traditional Rademacher complexity \citep{foundation}, these training-free metrics provide alternative methods to measure the complexity of DNNs when estimating the generalization gap of DNNs.

\subsection{Concrete Generalization Guarantees for Training-Free NAS}\label{sec:generalization-nas}

Since the $\Ls_{S}(\cdot)$ in our Theorem \ref{th:generalization} may also depend on the training-free metric $\gM$, it also needs to be taken into account when analyzing the generalization performance (or the generalization error $\Ls_{\gD}(\cdot)$) for training-free NAS methods.
To this end, in this section, we derive concrete generalization guarantees for NAS methods using training-free metrics by considering two different scenarios (i.e., the \emph{realizable} and \emph{non-realizable} scenarios) for the training error term $\Ls_{S}(\cdot)$ in Theorem \ref{th:generalization}, which finally give rise to principled interpretations for different training-free NAS methods \citep{te-nas, nasi, zero-cost}.

\paragraph{The realizable scenario.} 
Similar to \citep{foundation}, we assume that a zero training error (i.e., $\Ls_{S}(f_t) \rightarrow 0$ when $t$ is sufficiently large) can be achieved in the realizable scenario.
By further assuming that the condition number $\kappa$ in Theorem~\ref{th:generalization} is bounded by $\kappa_0$ for all candidate architectures in the search space, we can then derive the following generalization guarantee (Corollary \ref{corol:generalization-realizable}) for the realizable scenario.

\begin{corollary}\label{corol:generalization-realizable}
Under the conditions in Theorem~\ref{th:generalization}, for $f_t$ at convergence (i.e., $t \rightarrow \infty$) in the realizable scenario and for any training-free metric $\gM$ from Sec.~\ref{sec:metrics}, the following holds with a high probability,
\begin{equation*}
    \Ls_{\gD}(f_t) \leq \gO(1 / \gM) \ .
\end{equation*}
\end{corollary}

Corollary~\ref{corol:generalization-realizable} is obtained by introducing $\Ls_{S}(f_t)=0$ and $\kappa \leq \kappa_0$ into Theorem~\ref{th:generalization}. Importantly,~Corollary~\ref{corol:generalization-realizable} suggests that in the realizable scenario, the generalization error of DNNs is negatively correlated with the metrics from Sec.~\ref{sec:metrics}. That is, an architecture with a larger value of training-free metric $\gM$ generally achieves an improved generalization performance. This implies that in order to select well-performing architectures, we can simply maximize $\gM$ to find $\gA^* = \argmax_{\gA} \gM(\gA)$ where $\gA$ denotes any architecture in the search space.
Interestingly, this formulation aligns with the training-free NAS method from \citep{zero-cost}, which has made use of the metrics $\gM_{\normalfont\gradnorm},\gM_{\normalfont\snip}$ and $\gM_{\normalfont\grasp}$ to achieve good empirical performances. 
Therefore, our Corollary \ref{corol:generalization-realizable} provides a valid generalization guarantee and also a principled justification for the method from \citep{zero-cost}.

\paragraph{The non-realizable scenario.} 
In practice, different candidate architectures in a NAS search space typically have diverse non-zero training errors \citep{nasi} and $\kappa$ \citep{te-nas}. 
Therefore, the assumptions of the zero training error and the bounded $\kappa$ in the realizable scenario above may be impractical.
In light of this, we drop these two assumptions and derive the following generalization guarantee (Corollary~\ref{corol:generalization-non-realizable}) for the non-realizable scenario, which, interestingly, facilitates theoretically grounded interpretations for the training-free NAS methods from \citep{nasi, te-nas}.

\begin{corollary}\label{corol:generalization-non-realizable}
Under the conditions in Theorem \ref{th:generalization}, for any $f_t$ at time $t>0$ and any training-free metric $\gM$ from Sec.~\ref{sec:metrics} in the non-realizable~scenario, there exists a constant $C>0$ such that with a high probability,
\begin{equation*}
    \Ls_{\gD}(f_t) \leq \frac{1}{2}\left(m - \eta\gM^2 / C\right)^{2t} + \gO(\kappa/\gM) \ .
\end{equation*}
\end{corollary}

Its proof is given in Appendix~\ref{sec:proof:corol:generalization-non-realizable}. Notably, our Corollary \ref{corol:generalization-non-realizable} suggests that when $\gM \in [0, \sqrt{mC/\eta}]$, an architecture with a larger value of the metric $\gM$ will lead to a better generalization performance because such a model has both a faster convergence 
(i.e., the first term decreases faster w.r.t time $t$)
and a smaller generalization gap (i.e., the second term is smaller). 
Interestingly, \citet{nasi} have leveraged this insight to introduce the training-free metric of $\gM_{\tracenorm}$ with a constraint, which has achieved a higher correlation with the generalization performance of architectures than the metrics from \citep{zero-cost}.
This therefore implies that our Corollary~\ref{corol:generalization-non-realizable} followed by \citep{nasi} provides a better characterization of the generalization performance of architectures than Corollary~\ref{corol:generalization-realizable} followed by \citep{zero-cost} since the non-realizable scenario we have considered will be more realistic than the realizable scenario as explained above.
Meanwhile, Corollary~\ref{corol:generalization-non-realizable} also suggests that there exists a trade-off in terms of $\gM$ between the model convergence (i.e., the first term) and the generalization gap (i.e., the second term) when $\gM>\sqrt{mC/\eta}$, which surprisingly is similar to the empirically motivated trainability and expressivity trade-off in \citep{te-nas}.
In addition, Corollary~\ref{corol:generalization-non-realizable} also indicates that for architectures achieving similar values of $\gM$, the ones with smaller condition numbers $\kappa$ generally achieve better generalization performance.
Interestingly, such a result also aligns with the conclusion from \citep{te-nas}. Therefore, our Corollary~\ref{corol:generalization-non-realizable} also provides a principled justification for the training-free NAS method in \citep{te-nas}.

\subsection{Connection to Architecture Topology}\label{sec:topology}
Interestingly, we can prove that the condition number $\kappa$ in our Corollary~\ref{corol:generalization-non-realizable} is theoretically related to the architecture topology, i.e., whether the architecture is wide (and shallow) or deep (and narrow), to further support the practicality and the superiority of our Corollary \ref{corol:generalization-non-realizable}. In particular, inspired by the theoretical analysis from \citep{understand-nas}, we firstly analyze the eigenvalues of the NTK matrices of two different architecture topologies (i.e., wide vs. deep architectures), which gives us an insight into the difference between their corresponding $\kappa$. We mainly consider the following wide (i.e., $f$) and deep (i.e., $f'$) architecture illustrated in Figure~\ref{fig:topology}, respectively:
\begin{equation}
\begin{aligned}
f(\vx) = \textstyle \vone^{\top}\sum_{i=1}^L \rmW^{(i)}\vx\ , \; f'(\vx) = \vone^{\top}(\prod_{i=1}^L \rmW^{(i)})\vx \label{eq:wide-deep}
\end{aligned}
\end{equation}
where $\rmW^{(i)} \in \sR^{n \times n}$ for any $i \in \{1,\cdots,L\}$ and every element of $\rmW^{(i)}$ is independently initialized using the standard normal distribution. Here, $\vone$ denotes an $n$-dimensional vector with every element being one. Let $\rmTheta_0$ and $\rmTheta'_0$ be the NTK matrices of $f$ and $f'$ that are evaluated on the finite dataset $S=\{(\vx_i, y_i)\}_{i=1}^m$, respectively, we derive the following theorem:

\begin{theorem}\label{th:topology}
Let dataset $S$ be normalized using its statistical mean and covariance such that $\E[\vx]=0$ and $\rmX^{\top}\rmX = \rmI$ given $\rmX \triangleq [\vx_1 \vx_2\cdots\vx_m]$, we have
\begin{equation*}
\begin{aligned}
\rmTheta_0 = Ln\cdot \rmI\ , \; \E\left[\rmTheta'_0\right] = Ln^L  \cdot \rmI \ .
\end{aligned}
\end{equation*}
\end{theorem}

Its proof is in Appendix~\ref{sec:proof:th:topology}. 
Notably, Theorem~\ref{th:topology} shows that the NTK matrix of the wide architecture in \eqref{eq:wide-deep} is guaranteed to be a scaled identity matrix, whereas the NTK matrix of the deep architecture in \eqref{eq:wide-deep} is a scaled identity matrix \emph{only in expectation} over random initialization. Consequently, we always have $\kappa=1$ for the initialized wide architecture, while $\kappa>1$ with high probability for the initialized deep architecture. 
Also note that as we have discussed in Sec.~\ref{sec:generalization-nas}, our Corollary \ref{corol:generalization-non-realizable} shows that (given similar values of $\gM$) an architecture with a smaller $\kappa$ is likely to generalize better.
Therefore, this implies that wide architectures generally achieve better generalization performance than deep architectures (given similar values of $\gM$).
This, surprisingly, aligns with the findings from \citep{understand-nas} which shows that wide architectures are preferred in \emph{training-based NAS} due to their competitive performances in practice, thus further implying that our Corollary \ref{corol:generalization-non-realizable} is more practical and superior to our Corollary \ref{corol:generalization-realizable}. More interestingly, based on the definition of $\gM_{\tracenorm}$~\eqref{eq:metric:tracenorm}, Theorem~\ref{th:topology} also indicates that deep architectures are expected to have larger values of $\gM_{\tracenorm}$ (due to the larger scale of $\E\left[\rmTheta'_0\right]$ for deep architectures) and hence achieve larger model complexities than wide architectures.

\section{Hybrid Neural Architecture Search}\label{sec:hnas}
\subsection{A Unified Objective for Training-Free NAS}
Our theoretical understanding of training-free NAS in Sec.~\ref{sec:analyses} finally allows us to address the following question in a principled way: \emph{How can we consistently boost existing training-free NAS algorithms?}
Specifically, to realize this target, we propose to select well-performing architectures by minimizing the upper bound on the generalization error in Corollary~\ref{corol:generalization-non-realizable} given any training-free metric from Sec.~\ref{sec:metrics}.
We expect this choice to lead to improved performances over the method from \citep{zero-cost} because Corollary~\ref{corol:generalization-non-realizable} provides a more practical generalization guarantee for training-free NAS than Corollary~\ref{corol:generalization-realizable} followed by \citep{zero-cost} (Sec.~\ref{sec:generalization-nas}).
Formally, let $\gA$ be any architecture in the search space and let $\gM$ be any training-free metric from Sec.~\ref{sec:metrics}, then NAS problem can be formulated below in a unified manner:
\begin{equation}
\begin{aligned}
\min_{\gA} \frac{1}{2}\left(m - \eta\gM^2(\gA)/C\right)^{2t} + \gO\left(\kappa(\gA) / \gM(\gA)\right) \ . \label{eq:nas-non-realizable}
\end{aligned}
\end{equation}
We further reformulate \eqref{eq:nas-non-realizable} into the following form:
\begin{equation}
\begin{gathered}
\min_{\gA} \kappa(\gA) / \gM(\gA) + \mu F(\gM^2(\gA) - \nu) \label{eq:nas-final}
\end{gathered}
\end{equation}
where $F(x) \triangleq x^{2t}$, and $\mu$ and $\nu$ are hyperparameters we introduced to absorb the impact of all other parameters in \eqref{eq:nas-non-realizable}. Compared with the diverse form of NAS objectives in \citep{te-nas, nasi, zero-cost}, our \eqref{eq:nas-final} presents a non-trivial unified form of NAS objectives for all the training-free metrics from Sec.~\ref{sec:metrics}, making it easier for practitioners to deploy NAS with different types of evaluated training-free metrics. Our NAS objective in \eqref{eq:nas-final} is a natural consequence of our generalization guarantee in Corollary~\ref{corol:generalization-non-realizable} and therefore will be more theoretically grounded, in contrast to the heuristic objective in \citep{te-nas}. Moreover, our \eqref{eq:nas-final} advances the training-free NAS method based on $\gM_{\tracenorm}$ from \citep{nasi}, because our \eqref{eq:nas-final} \emph{(a)} is derived 
from the generalization error instead of the training error (that is followed by \citep{nasi}) of DNNs, which therefore will be more sound and practical, \emph{(b)} have unified all the gradient-based training-free metrics from Sec.~\ref{sec:metrics}, and \emph{(c)} have considered the impact of condition number $\kappa$ which is shown to be critical in practice (see our Appendix~\ref{sec:app:generalization}). Above all, our unified NAS objective in \eqref{eq:nas-final} is expected to be able to lead to improved performances over other existing training-free NAS methods.

\subsection{Optimization and Search Algorithm}
\begin{algorithm}[t]
  \caption{Hybrid Neural Architecture Search (HNAS)}
  \label{alg:hnas}
\begin{algorithmic}[1]
  \STATE {\bfseries Input:} Training and validation dataset, metric $\gM$ evaluated on architecture pool $\gP$, $F(\cdot)$ for \eqref{eq:nas-final}, evaluation history $\gH_{0}=\varnothing$, a BO algorithm $\gB$, number of iterations/queries $K$
  \FOR{iteration $k=1, \ldots, K$}
  \STATE Choose $\mu_k, \nu_k$ using the BO algorithm $\gB$
  \STATE Obtain the optimal candidate $\gA_k^*$ in $\gP$ by solving \eqref{eq:nas-final}
  \STATE Evaluate the validation performance of $\gA_k^*$, e.g., $\gL_{\text{val}}(\gA_k^*)$ after training $\gA_k^*$
  \STATE Update the GP surrogate in the BO algorithm $\gB$ using the evaluation history $\gH_{k} = \gH_{k-1} \bigcup \left\{\left((\mu_k, \nu_k), \gL_{\text{val}}(\gA_k^*)\right)\right\}$
  \ENDFOR
  \STATE Select the final $\gA^*$ with the best validation performance, e.g., $\gA^* = \argmin_{\gA \in \{\gA_k^*\}_{k=1}^K} \gL_{\text{val}}(\gA)$
\end{algorithmic}
\end{algorithm}

Our theoretically motivated NAS objective in \eqref{eq:nas-final} has unified all training-free metrics from Sec.~\ref{sec:metrics} and improved over existing training-free NAS methods. However, its practical deployment requires the determination of the hyperparameters $\mu$ and $\nu$,\footnote{Of note, we usually fix $t=1$, which is already reasonably good for $F(\cdot)$. So, the practical deployment of \eqref{eq:nas-final} will mainly be affected by the choice of $\mu$ and $\nu$.} which can be non-trivial in practice.
To this end, we further introduce \emph{Bayesian optimization} (BO) \citep{bo} to optimize the hyperparameters $\mu$ and $\nu$ in order \emph{to maximize the true validation performance} of the architectures selected by different $\mu$ and $\nu$.
In particular, BO uses a Gaussian process (GP) as a surrogate to model the objective function (i.e., the validation performance here) in order to sequentially choose the queried inputs (i.e., the values of $\mu$ and $\nu$).
This finally completes our theoretically grounded NAS framework called \emph{hybrid NAS} (HNAS), which not only novelly unifies all training-free metrics from Sec.~\ref{sec:metrics} but also boosts NAS algorithms based on these training-free metrics in a principled way (Algorithm \ref{alg:hnas}).

Specifically, in every iteration $k$ of HNAS, we firstly select the optimal candidate $\gA_k^*$ by maximizing our training-free NAS objective in \eqref{eq:nas-final} using the values of $\mu$ and $\nu$ queried by the BO algorithm in the current iteration (line 3-4 of Algorithm \ref{alg:hnas}). Next, we evaluate the validation performance of $\gA_k^*$ (e.g., validation error $\gL_{\text{val}}(\gA_k^*)$) and then use it to update the GP surrogate that is applied in the BO algorithm (line 5-6 of Algorithm \ref{alg:hnas}), which then will be used to choose the values of $\mu$ and $\nu$ in the next iteration.
After HNAS completes, the final selected architecture is chosen as the one achieving the best validation performance among all the optimal candidates, e.g., $\gA^* = \argmin_{\gA \in \{\gA_k^*\}_{k=1}^K} \gL_{\text{val}}(\gA)$ (see Appendix \ref{sec:app:opt-details} for more optimization details of Algorithm \ref{alg:hnas}).
Thanks to the utilization of validation performance as the objective for BO, our HNAS is expected to be able to enjoy the advantages of both training-free (i.e., the superior search efficiency) and training-based NAS (i.e., the remarkable search effectiveness) as supported by our extensive empirical results in Sec.~\ref{sec:exp:hnas}. In addition, by novelly introducing BO to optimize the low-dimensional continuous hyperparameters $\mu$ and $\nu$ rather than the high-dimensional discrete architectural hyperparameters in the NAS search space, HNAS is able to avoid the issues of high-dimensional discrete optimization that standard BO algorithms usually attain when they are directly applied for NAS \citep{bonas}, allowing HNAS to be more efficient and effective in practice as empirically supported in our Sec.~\ref{sec:exp:hnas}.

\section{Experiments}\label{sec:exps}
\subsection{Connections among Training-Free Metrics}\label{sec:exp:connection}

\begin{figure}[t]
\centering
\begin{tabular}{cc}
    \hspace{-2mm}\includegraphics[width=0.50\columnwidth]{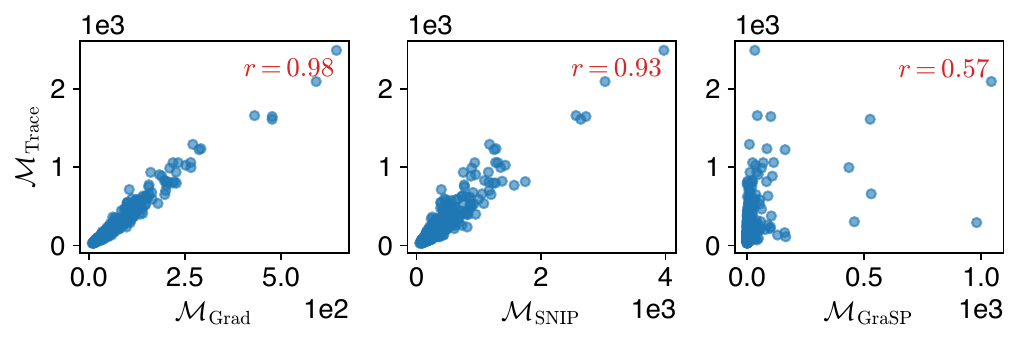} & \hspace{-4mm}\includegraphics[width=0.50\columnwidth]{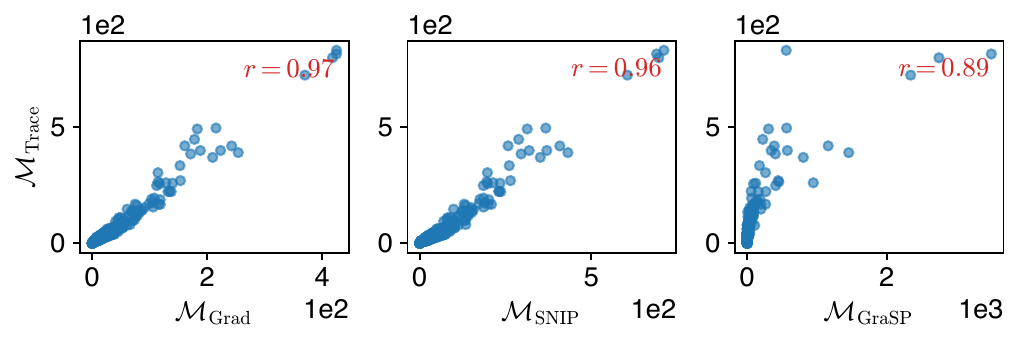} \\
    {(a) NAS-Bench-101} & {(b) NAS-Bench-201}
\end{tabular}
\caption{Spearman correlation between $\gM_{\tracenorm}$ and other training-free metrics from Sec.~\ref{sec:metrics}, which are evaluated in NAS-Bench-101/201. The correlation coefficient $r$ is given in the corner of each plot.}
\label{fig:connection}
\vskip -0.15in
\end{figure}

We firstly validate the theoretical connections between $\gM_{\tracenorm}$ and other training-free metrics from Sec. \ref{sec:metrics} by examining their Spearman correlations for all architectures in NAS-Bench-101 \citep{nasbench101} and NAS-Bench-201 \citep{nasbench201} with CIFAR-10 \citep{cifar}. Figure~\ref{fig:connection} illustrates the result where all these training-free metrics are evaluated using a batch (with size 64) of sampled data following that of \citep{zero-cost}. Of note, we will follow the same approach to evaluate these training-free metrics in our following sections.
The results in Figure~\ref{fig:connection} show that $\gM_{\tracenorm}$ and other training-free metrics from Sec.~\ref{sec:metrics} are indeed highly correlated since they consistently achieve high positive correlations in different search spaces. These empirical results actually align with the interpretation of our Theorem~\ref{th:connection} (Sec.~\ref{sec:connection}). Moreover, the correlation between any two training-free~metrics from Sec.~\ref{sec:metrics} is in Appendix~\ref{sec:app:connection}, which further verifies the connection among all these training-free metrics. Above all, in addition to the theoretical justification in our Theorem~\ref{th:connection}, our empirical results have also supported the connections among all the training-free metrics from Sec.~\ref{sec:metrics}.

\begin{figure}[!t]
\begin{minipage}[t]{0.50\linewidth}
\renewcommand\multirowsetup{\centering}
\captionof{table}{Correlation coefficients between the test errors evaluated on CIFAR-10 and the generalization guarantees in Sec.~\ref{sec:generalization-nas} for the architectures in NAS-Bench-101/201.}
\label{tab:generalization}
\centering
\resizebox{\columnwidth}{!}{
\begin{tabular}{lcccc}
\toprule
\multirow{2}{*}{\textbf{Metric}} & \multicolumn{2}{c}{\textbf{NAS-Bench-101}} &
\multicolumn{2}{c}{\textbf{NAS-Bench-201}} \\
\cmidrule(l){2-3} \cmidrule(l){4-5} 
& Spearman & Kendall's Tau & Spearman & Kendall's Tau \\
\midrule
\multicolumn{5}{c}{\textbf{Realizable scenario}} \\
$\gM_{\gradnorm}$ & $-$0.25 & $-$0.17 & 0.64 & 0.47 \\
$\gM_{\snip}$ & $-$0.21 & $-$0.15 & 0.64 & 0.47 \\
$\gM_{\grasp}$ & $-$0.45 & $-$0.31 & 0.57 & 0.40 \\
$\gM_{\tracenorm}$ & $-$0.30 & $-$0.21 & 0.54 & 0.39 \\
\midrule
\multicolumn{5}{c}{\textbf{Non-realizable scenario}} \\
$\gM_{\gradnorm}$ & $\ms$0.35 & $\ms$0.23 & 0.75 & 0.56 \\
$\gM_{\snip}$ & $\ms$0.37 & $\ms$0.25 & 0.75 & 0.56 \\
$\gM_{\grasp}$ & $\ms$0.46 & $\ms$0.32 & 0.69 & 0.50\\
$\gM_{\tracenorm}$ & $\ms$0.33 & $\ms$0.23 & 0.70 & 0.51 \\
\bottomrule
\end{tabular}
}
\vskip -0.1in
\end{minipage}\hfill
\begin{minipage}[t]{0.49\linewidth}
\renewcommand\multirowsetup{\centering}
\captionof{table}{Comparison of topology, $\gM_{\tracenorm}$ \& $\kappa$ of different architectures. The topology width/depth of each architecture is followed by the maximum value in the search space (separated by a slash). 
}
\label{tab:topology}
\centering
\resizebox{\columnwidth}{!}{
\begin{tabular}{lcccc}
\toprule
\multirow{2}{*}{\textbf{Architecture}} & \multicolumn{2}{c}{\textbf{Topology}} & \multirow{2}{*}{$\gM_{\tracenorm}$}  & 
\multirow{2}{*}{$\kappa$}\\
\cmidrule(l){2-3}
 & Width & Depth & & \\
\midrule
NASNet & 5.0/5.0 & 2/6 & 31$\pm$2 & 118$\pm$41\\
AmoebaNet & 4.0/5.0 & 4/6 & 36$\pm$2 & 110$\pm$39 \\
ENAS & 5.0/5.0 & 2/6 & 36$\pm$2 & $\ms$98$\pm$33 \\
DARTS & 3.5/4.0 & 3/5 & 33$\pm$2 & 122$\pm$58\\
SNAS & 4.0/4.0 & 2/5 & 31$\pm$2 &  126$\pm$47 \\
\midrule
WIDE & 4.0/4.0 & 2/5 & 27$\pm$1 & 141$\pm$36 \\
DEEP & \textbf{1.5}/4.0 & \textbf{5}/5 & \textbf{131}$\pm$\textbf{16} & $\ms$\textbf{209}$\pm$\textbf{107} \\
\bottomrule
\end{tabular}
}
\vskip -0.1in
\end{minipage}
\vskip -0.15in
\end{figure}

\subsection{Generalization Guarantees for Training-Free NAS}\label{sec:exp:generalization}
We then demonstrate the validity of our generalization guarantees for training-free NAS (Sec. \ref{sec:generalization-nas}) by examining the correlation between the generalization bound in the realizable (Corollary \ref{corol:generalization-realizable}) or non-realizable (Corollary \ref{corol:generalization-non-realizable}) scenario and the test errors of architectures in NAS-Bench-101/201. Similar to HNAS (Algorithm \ref{alg:hnas}), we employ BO with a sufficiently large number of iterations (e.g., hundreds of iterations) to determine the non-trivial parameters in Corollary \ref{corol:generalization-non-realizable}. Table \ref{tab:generalization} summarizes the results on CIFAR-10 where a higher positive correlation implies a better agreement between our generalization guarantee and the generalization performance of architectures. Notably, the generalization bound in the realizable scenario performs a compelling characterization of the test errors in NAS-Bench-201 with relatively high positive correlations, whereas it fails to provide a precise characterization in a larger search space, i.e., NAS-Bench-101. Remarkably, our generalization bound in the non-realizable scenario is able to perform consistent improvement over it by obtaining higher positive correlations. These results imply that the Corollary \ref{corol:generalization-realizable} may only provide a good characterization for training-free NAS in certain cases (e.g., in the small-scale search space NAS-Bench-201), whereas our Corollary \ref{corol:generalization-non-realizable} generally is more valid and robust in practice. As a consequence, our \eqref{eq:nas-non-realizable} following Corollary \ref{corol:generalization-non-realizable} should be able to improve over the NAS objective following Corollary \ref{corol:generalization-realizable} as we have justified in Sec. \ref{sec:hnas}. Interestingly, the comparable results achieved by all training-free metrics from Sec. \ref{sec:metrics} again validate the connections among these metrics (Theorem \ref{th:connection}). Moreover, our additional results in Appendix \ref{sec:app:generalization} further confirm the validity and practicality of our generalization guarantees for training-free NAS.

\subsection{Connection to Architecture Topology}

To support the theoretical connections between architecture topology (wide vs. deep) and the value of training-free metric $\gM_{\tracenorm}$ as well as the condition number $\kappa$ shown in Sec. \ref{sec:topology}, we compare the topology width/depth, $\gM_{\tracenorm}$ and $\kappa$ of the architectures selected by different SOTA training-based NAS algorithms in the DARTS search space, including NASNet \citep{nasnet}, AmoebaNet \citep{amoebanet}, ENAS \citep{enas}, DARTS \citep{darts}, and SNAS \citep{snas}. Table~\ref{tab:topology} summarizes the results where we apply the same definition of topology width/depth in \citep{understand-nas} (refer to \citep{understand-nas} for more details).
We also include the widest (called WIDE) and the deepest (called DEEP) architecture in the DARTS search space into this comparison. As shown in our Table~\ref{tab:topology}, wide architectures (i.e., all architectures except DEEP) consistently achieve lower condition number $\kappa$ and smaller values of $\gM_{\tracenorm}$ than deep architecture (i.e., DEEP), which aligns with our theoretical insights in Sec.~\ref{sec:topology}.

\subsection{Effectiveness and Efficiency of HNAS}\label{sec:exp:hnas}
\begin{table*}[t!]
\caption{Comparison of NAS algorithms in NAS-Bench-201. The result of HNAS is reported with the mean and standard deviation of 5 independent searches and its search costs are evaluated on a Nvidia 1080Ti. C \& D in the last column denote continuous and discrete search space, respectively.}
\label{tab:sota-nasbench201}
\centering
\resizebox{\textwidth}{!}{
\begin{threeparttable}
\begin{tabular}{lccccccc}
\toprule
\multirow{2}{*}{\textbf{Algorithm}} & \multicolumn{3}{c}{\textbf{Test Accuracy (\%)}} &
\multirow{2}{*}{\textbf{Cost}} &
\multirow{2}{*}{\textbf{Method}} &
\multirow{2}{*}{\textbf{Applicable}} \\
\cmidrule(l){2-4} 
& C10 & C100 & IN-16 & (GPU Sec.) & & \textbf{Space}\\
\midrule 
ResNet \citep{resnet} & 93.97 & 70.86 & 43.63 & - & manual & - \\
\midrule
REA$^{\dagger}$ & 93.92$\pm$0.30 & 71.84$\pm$0.99 & 45.15$\pm$0.89 & 12000 & evolution & C \& D \\
RS (w/o sharing)$^{\dagger}$ & 93.70$\pm$0.36 & 71.04$\pm$1.07 & 44.57$\pm$1.25 & 12000 & random & C \& D \\
REINFORCE$^{\dagger}$ & 93.85$\pm$0.37 & 71.71$\pm$1.09 & 45.24$\pm$1.18 & 12000 & RL & C \& D \\
BOHB$^{\dagger}$ & 93.61$\pm$0.52 & 70.85$\pm$1.28 & 44.42$\pm$1.49 & 12000 & BO+bandit & C \& D \\
\midrule
ENAS$^{\ddagger}$ \citep{enas} & 93.76$\pm$0.00 & 71.11$\pm$0.00 & 41.44$\pm$0.00 & 15120 & RL & C \\
DARTS (1st)$^{\ddagger}$ \citep{darts} & 54.30$\pm$0.00 & 15.61$\pm$0.00 & 16.32$\pm$0.00 & 16281 & gradient & C \\
DARTS (2nd)$^{\ddagger}$ \citep{darts} & 54.30$\pm$0.00 & 15.61$\pm$0.00 & 16.32$\pm$0.00 & 43277 & gradient & C \\
GDAS$^{\ddagger}$ \citep{gdas} & 93.44$\pm$0.06 & 70.61$\pm$0.21 & 42.23$\pm$0.25 & 8640 & gradient & C \\
DrNAS$^{\sharp}$ \citep{drnas} & 93.98$\pm$0.58 & 72.31$\pm$1.70 & 44.02$\pm$3.24 & 14887 & gradient & C \\
\midrule
NASWOT \citep{naswot} & 92.96$\pm$0.81 & 69.98$\pm$1.22 &  44.44$\pm$2.10 & 306 & training-free & C \& D\\
TE-NAS \citep{te-nas} & 93.90$\pm$0.47 & 71.24$\pm$0.56 & 42.38$\pm$0.46 & 1558 & training-free & C \\
KNAS \citep{knas} & 93.05 & 68.91 & 34.11 & 4200 & training-free & C \& D \\
NASI \citep{nasi} & 93.55$\pm$0.10 & 71.20$\pm$0.14 & 44.84$\pm$1.41 & 120 & training-free & C \\
GradSign \citep{gradsign} & 93.31$\pm$0.47 & 70.33$\pm$1.28 & 42.42$\pm$2.81 & - & training-free & C \& D \\
\midrule
HNAS ($\gM_{\gradnorm}$) & \textbf{94.04}$\pm$0.21 & 71.75$\pm$1.04 & \textbf{45.91}$\pm$0.88 & 3010 & hybrid & C \& D \\
HNAS ($\gM_{\snip}$) & \textbf{93.94}$\pm$0.02 & 71.49$\pm$0.11 & \textbf{46.07}$\pm$0.14 & 2976 & hybrid & C \& D \\
HNAS ($\gM_{\grasp}$) & \textbf{94.13}$\pm$0.13 & \textbf{72.59}$\pm$0.82 & \textbf{46.24}$\pm$0.38 & 3148 & hybrid & C \& D \\
HNAS ($\gM_{\tracenorm}$) & \textbf{94.07}$\pm$0.10 & \textbf{72.30}$\pm$0.70 & \textbf{45.93}$\pm$0.37 & 3006 & hybrid & C \& D \\
\midrule
\textbf{Optimal} & 94.37 & 73.51 & 47.31 & - & - & -\\
\bottomrule
\end{tabular}
\begin{tablenotes}\footnotesize
    \item[$\dagger$] Reported by \citet{nasbench201}.
    \item[$\ddagger$] Re-evaluated using the codes provided by \citet{nasbench201}.
    \item[$\sharp$] Re-evaluated under a comparable search budget as other training-based NAS algorithms with first-order optimization, e.g., ENAS and DARTS (1st). Note that this search budget is smaller than the one reported in its original paper and hence will lead to decreased search performances.
\end{tablenotes}
\end{threeparttable}
} 
\vskip -0.1in
\end{table*}

To justify that our theoretically motivated HNAS is able to enjoy the advantages of both training-free (i,e., the superior search efficiency) and training-based (i.e., the remarkable search effectiveness) NAS, we compare it with other baselines in NAS-Bench-201 (Table~\ref{tab:sota-nasbench201}). We refer to Appendix~\ref{sec:app:exp-details} for our experimental details. As summarized in Table~\ref{tab:sota-nasbench201}, HNAS, surprisingly, advances both training-based and training-free baselines by consistently selecting architectures achieving the best performances, leading to smaller gaps toward the optimal test errors in the search space. Meanwhile, HNAS requires at most 13$\times$ lower search costs than training-based NAS algorithms, which is even smaller than the training-free baseline KNAS. Moreover, thanks to the superior evaluation efficiency of training-free metrics, HNAS can be deployed efficiently in not only continuous (where search space is represented as a supernet) but also discrete search space. As for NAS under limited search budgets (Figure~\ref{fig:sota-nasbench201}), HNAS also advances all other baselines by achieving improved search efficiency and effectiveness. Appendix~\ref{sec:app:sota-darts} further includes the impressive search results achieved by HNAS on CIFAR-10/100 and ImageNet in the DARTS search space.
Overall, our HNAS is indeed able to enjoy the advantages of both training-free (i.e., the superior search efficiency) and training-based NAS (i.e., the remarkable search effectiveness), which consistently boosts existing training-free NAS methods.

\begin{figure}[t]
\centering
\hspace{-2mm}\includegraphics[width=\columnwidth]{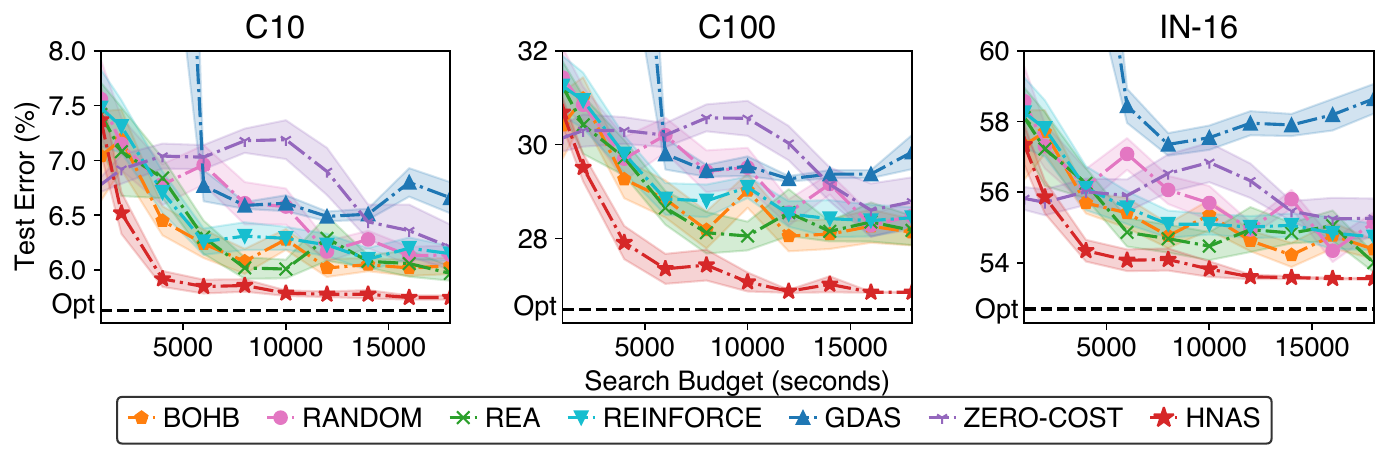}
\caption{Comparison between HNAS ($\gM_{\tracenorm}$) and other NAS baselines in NAS-Bench-201 under varying search budgets. Here, the ZERO-COST method is borrowed from \citep{zero-cost} by using $\gM_{\tracenorm}$. Note that each algorithm is reported with the mean and standard error of ten independent searches, and the black dashed line in each plot denotes the the minimal (optimal) test error that can be achieved by the architectures in NAS-Bench-201 on the corresponding dataset.
}
\label{fig:sota-nasbench201}
\vskip -0.15in
\end{figure}

\section{Conclusion \& Discussion}
This paper performs a unified theoretical analysis of NAS algorithms using gradient-based training-free metrics, which allows us to \emph{(a)} theoretically unveil the connections among these training-free metrics, \emph{(b)} provide theoretical guarantees for the empirically observed compelling performance of these training-free NAS algorithms, and \emph{(c)} exploit these theoretical understandings to develop a novel framework called HNAS that can consistently boost existing training-free NAS. 
We expect that our theoretical understanding to provide valuable prior knowledge for the design of training-free metrics and NAS search space in the future.
Moreover,
we expect our theoretical analyses for DNNs to be capable of inspiring more theoretical understanding and improvement over existing machine learning algorithms that are based on DNNs, e.g., the recent training-free data valuation algorithm \citep{davinz}. In addition, the impressive performance achieved by our HNAS framework is expected to be able to encourage more attention to the integration of training-free and training-based approaches in other fields in order to enjoy the advantages of these two types of methods simultaneously. 

\begin{ack}
    This research/project is supported by the National Research Foundation Singapore and DSO National Laboratories under the AI Singapore Programme (AISG Award No: AISG$2$-RP-$2020$-$018$) and by A*STAR under its RIE$2020$ Advanced Manufacturing and Engineering (AME) Programmatic Funds (Award A$20$H$6$b$0151$). 
\end{ack}

\newpage
\bibliographystyle{unsrtnat}
\bibliography{workspace/reference}

\newpage
\section*{Checklist}

\begin{enumerate}

\item For all authors...
\begin{enumerate}
  \item Do the main claims made in the abstract and introduction accurately reflect the paper's contributions and scope?
    \answerYes{}
  \item Did you describe the limitations of your work?
    \answerYes{} See Sec~\ref{sec:ntk}, i.e., the simplified analysis on the fully connected neural networks with scalar output.
  \item Did you discuss any potential negative societal impacts of your work?
    \answerNo{} I do not see any potential negative societal impacts of this paper.
  \item Have you read the ethics review guidelines and ensured that your paper conforms to them?
    \answerYes{}
\end{enumerate}

\item If you are including theoretical results...
\begin{enumerate}
  \item Did you state the full set of assumptions of all theoretical results?
    \answerYes{}
\item Did you include complete proofs of all theoretical results?
    \answerYes{}
\end{enumerate}

\item If you ran experiments...
\begin{enumerate}
  \item Did you include the code, data, and instructions needed to reproduce the main experimental results (either in the supplemental material or as a URL)?
    \answerYes{}
  \item Did you specify all the training details (e.g., data splits, hyperparameters, how they were chosen)?
    \answerYes{}
        \item Did you report error bars (e.g., with respect to the random seed after running experiments multiple times)?
    \answerYes{}
        \item Did you include the total amount of compute and the type of resources used (e.g., type of GPUs, internal cluster, or cloud provider)?
    \answerYes{}
\end{enumerate}

\item If you are using existing assets (e.g., code, data, models) or curating/releasing new assets...
\begin{enumerate}
  \item If your work uses existing assets, did you cite the creators?
    \answerYes{}
  \item Did you mention the license of the assets?
    \answerNo{} All assets I have used are public.
  \item Did you include any new assets either in the supplemental material or as a URL?
    \answerNA{}
  \item Did you discuss whether and how consent was obtained from people whose data you're using/curating?
    \answerNA{}
  \item Did you discuss whether the data you are using/curating contains personally identifiable information or offensive content?
    \answerNA{}
\end{enumerate}

\item If you used crowdsourcing or conducted research with human subjects...
\begin{enumerate}
  \item Did you include the full text of instructions given to participants and screenshots, if applicable?
    \answerNA{}
  \item Did you describe any potential participant risks, with links to Institutional Review Board (IRB) approvals, if applicable?
    \answerNA{}
  \item Did you include the estimated hourly wage paid to participants and the total amount spent on participant compensation?
    \answerNA{}
\end{enumerate}

\end{enumerate}

\appendix
\begin{appendices}
\onecolumn

\section{Proofs}
Throughout the proofs of this paper, we use lower-case bold-faced symbols to denote column vectors (e.g., $\vx$), and upper-case bold-faced symbols to represent matrices (e.g., $\rmA$).

\subsection{Proof of Theorem~\ref{th:connection}}\label{sec:proof:connection}

\paragraph{Connecting $\gM_{\normalfont\gradnorm}$ with $\gM_{\normalfont\tracenorm}$.} As the loss function $\ell(\cdot, \cdot)$ is assumed to be $\beta$-Lipschitz continuous in the first argument, the following holds based on the notations in Sec.~\ref{sec:notations}:
\begin{equation}
\begin{aligned}
    \gM_{\gradnorm} &\stackrel{(a)}{=} \left\|\frac{1}{m}\sum_{i=1}^m \nabla_{\vtheta}\ell(f(\vx_i, \vtheta_0), y_i)\right\|_2 \\
    &\stackrel{(b)}{\leq} \frac{1}{m}\sum_{i=1}^m \left\|\nabla_{\vtheta}\ell(f(\vx_i, \vtheta_0), y_i)\right\|_2 \\
    &\stackrel{(c)}{\leq} \frac{1}{m}\sum_{i=1}^m\big|\nabla_{f}\ell(f(\vx_i, \vtheta_0), y_i)\big|\|\nabla_{\vtheta}f(\vx_i, \vtheta_0)\|_2 \\
    &\stackrel{(d)}{\leq} \frac{\beta}{m}\sum_{i=1}^m\|\nabla_{\vtheta}f(\vx_i, \vtheta_0)\|_2 \\
    &\stackrel{(e)}{\leq} \frac{\beta}{m}\sqrt{m\sum_{i=1}^m\|\nabla_{\vtheta}f(\vx_i, \vtheta_0)\|_2^2} \\
    &\stackrel{(f)}{=} \beta \gM_{\tracenorm} \label{eq:grad&trace}
\end{aligned}
\end{equation}
where we let $\nabla_{f}\ell(f(\vx_i, \vtheta_0), y_i)$ be the gradient of the output of DNN model $f$. Note that $(a)$ follows from the definition of $\gM_{\gradnorm}$ in Sec.~\ref{sec:metrics} and $(b)$ derives from the Minkowski inequality. In addition, $(d)$ is from the definition of Lipschitz continuity and $(e)$ follows from the Cauchy-Schwarz inequality. Finally, $(f)$ is based on the definition of NTK matrix in Sec.~\ref{sec:ntk} and $\gM_{\tracenorm}$ in Sec.~\ref{sec:metrics}, i.e., 
\begin{equation}
\begin{aligned}
    \gM_{\tracenorm} = \sqrt{\frac{1}{m}\|\rmTheta_0\|_{\tr}} = \sqrt{\frac{1}{m}\sum_{i=1}^m \|\nabla_{\vtheta}f(\vx_i, \vtheta_0)\|_2^2} \ .
\end{aligned}
\end{equation}

Let $C_1 \triangleq \beta$, we then have
\begin{equation}
    \gM_{\gradnorm} \leq C_1\gM_{\tracenorm} \ .
\end{equation}

\paragraph{Connecting $\gM_{\normalfont\snip}$ with $\gM_{\normalfont\gradnorm}$.} We firstly introduce the following lemma.

\begin{lemma}[\citet{laurent2000adaptive}]
    If $\rx_1,\cdots,\rx_k$ are independent standard normal random variables, for $\ry=\sum_{i=1}^k \rx_i^2$ and any $\eps$, 
    \begin{equation*}
         \sP(\ry - k \geq 2\sqrt{k\eps} + 2\eps) \leq \exp(-\eps) \ .
    \end{equation*}
\end{lemma}
Following the common practice in \citep{ntk, exact-ntk}, each element of $\vtheta_0 \in \sR^d$ follows from the standard normal distribution independently. We therefore can bound $\|\vtheta_0\|_2^2$ using the lemma above. Specifically, let $\delta=\exp(-\eps) \in (0,1)$, with probability at least $1-\delta$ over random initialization, we have:
\begin{equation}
\begin{aligned}
    \|\vtheta_0\|_2^2 \leq d + 2\sqrt{d\ln\frac{1}{\delta}}+2\ln\frac{1}{\delta} \ .
\end{aligned}
\end{equation}

Using the results above and following the definition of $\gM_{\gradnorm}$, with probability at least $1-\delta$ over random initialization, we have
\begin{equation}
\begin{aligned}
    \gM_{\snip} &= \frac{1}{m}\sum_{i=1}^m\left|\vtheta_0^{\top} \nabla_{\vtheta}\gL(f(\vx_i, \vtheta_0), y_i)\right| \\
    &\leq \frac{1}{m}\sum_{i=1}^m \|\vtheta_0\|_2\|\nabla_{\vtheta}\ell(f(\vx_i, \vtheta_0), y_i)\|_2 \\
    &\leq \sqrt{d + 2\sqrt{d\ln\frac{1}{\delta}}+2\ln\frac{1}{\delta}} \cdot \frac{1}{m}\sum_{i=1}^m \left\|\nabla_{\vtheta}\ell(f(\vx_i, \vtheta_0), y_i)\right\|_2 \\
    & \leq \beta\sqrt{d + 2\sqrt{d\ln\frac{1}{\delta}}+2\ln\frac{1}{\delta}} \gM_{\tracenorm} \ .
\end{aligned}
\end{equation}
The last inequality follows from the same derivation in \eqref{eq:grad&trace}. Let $C_2 \triangleq \beta\sqrt{d + 2\sqrt{d\ln\frac{1}{\delta}}+2\ln\frac{1}{\delta}}$, the following then holds with a high probability (i.e., at least $1-\delta$),
\begin{equation}
    \gM_{\snip} \leq C_2\gM_{\tracenorm} \ .
\end{equation}

\paragraph{Connecting $\gM_{\normalfont\grasp}$ and $\gM_{\normalfont\gradnorm}$.} 
We firstly introduce the following lemma adapted from \citep{ntk-linear}.
\begin{lemma}[Lemma 1 in \citep{ntk-linear}]\label{th:continuity&smoothness}
     Let $\delta \in (0,1)$. There exist the constant $\rho_1, \rho_2>0$ such that for any $r>0$, $\vtheta, \vtheta' \in B(\vtheta_0, r/\sqrt{n})$ and any input $\vx$ within the dataset, with probability at least $1-\delta$ over random initialization, we have
    \begin{equation*}
    \begin{aligned}
        \left\|\nabla_{\vtheta}f(\vx, \vtheta)\right\|_2 &\leq \rho_1 \\
        \left\|\nabla_{\vtheta}f(\vx, \vtheta) - \nabla_{\vtheta'}f(\vx, \vtheta')\right\|_2 &\leq \rho_2\left\|\vtheta - \vtheta'\right\|_2 
    \end{aligned}
    \end{equation*}
    where $B(\vtheta_0, r/\sqrt{n}) \triangleq \{\vtheta \mid \|\vtheta - \vtheta_0\| \leq r/\sqrt{n}\}$.
\end{lemma}

To ease the notation, we use $\nabla_f \ell(f(\vx_i, \vtheta_0), y_i)$ to denote the gradient of the output (i.e., $f(\vx_i, \vtheta_0)$) from the DNN model $f$. According to the definition of Hessian matrix, $\rmH_i$ applied in $\gM_{\grasp}$ can be computed as
\begin{equation}
\begin{aligned}
    \rmH_i &= \nabla^2_{\vtheta_0}\ell(f(\vx_i, \vtheta_0), y_i) \\
    &=\nabla_{\vtheta} \left[\nabla_{f}\ell(f(\vx_i, \vtheta_0), y_i)\nabla_{\vtheta}f(\vx_i, \vtheta_0) \right] \\
    &= \nabla_f^2 \ell(f(\vx_i, \vtheta_0), y_i) \nabla_{\vtheta}f(\vx_i, \vtheta_0)\nabla_{\vtheta}f(\vx_i, \vtheta_0)^{\top} + \nabla_{f}\ell(f(\vx_i, \vtheta_0), y_i)\nabla_{\vtheta}^2f(\vx_i, \vtheta_0) \ .
\end{aligned}
\end{equation}

Since $\ell(\cdot, \cdot)$ is assumed to be $\gamma$-Lipschitz smooth and $\beta$-Lipschitz continuous in the first argument, we can then bound the operator norm of this hessian matrix $\rmH_i$ induced by the input $\vx_i$ in the dataset with
\begin{equation}
\begin{aligned}
    \left\|\rmH_i\right\|_2 &= \left\|\nabla_f^2 \ell(f(\vx_i, \vtheta_0), y_i) \nabla_{\vtheta}f(\vx_i, \vtheta_0)\nabla_{\vtheta}f(\vx_i, \vtheta_0)^{\top} + \nabla_{f}\ell(f(\vx_i, \vtheta_0), y_i)\nabla_{\vtheta}^2f(\vx_i, \vtheta_0)\right\|_2 \\
    &\leq \left|\nabla_f^2 \ell(f(\vx_i, \vtheta_0), y_i)\right| \left\| \nabla_{\vtheta}f(\vx_i, \vtheta_0)\nabla_{\vtheta}f(\vx_i, \vtheta_0)^{\top}\right\|_2 + \left|\nabla_{f}\ell(f(\vx_i, \vtheta_0), y_i)\right|\left\|\nabla_{\vtheta}^2f(\vx_i, \vtheta_0)\right\|_2 \\
    &\leq \gamma\left\|\nabla_{\vtheta}f(\vx_i, \vtheta_0)\nabla_{\vtheta}f(\vx_i, \vtheta_0)^{\top}\right\|_2 + \beta\left\|\nabla_{\vtheta}^2f(\vx_i, \vtheta_0)\right\|_2 \\
    &= \gamma \left\|\nabla_{\vtheta}f(\vx_i, \vtheta_0)\right\|_2^2 + \beta\left\|\nabla_{\vtheta}^2f(\vx_i, \vtheta_0)\right\|_2 \\
    &\leq \gamma\rho_1^2 + \beta \rho_2 \label{eq:hessian-norm}
\end{aligned}
\end{equation}
where the last inequality results from Lemma \ref{th:continuity&smoothness} and is satisfied with probability at least $1-\delta$ over random initialization.

Finally, let $\delta' \in (0,1)$, based on the definition of $\gM_{\grasp}$, the following then holds with probability at least $1-(m+1)\delta'$ over random initialization,
\begin{equation}
\begin{aligned}
    \gM_{\grasp} &= \frac{1}{m}\left|\sum_{i=1}^m  \vtheta_0^{\top}(\rmH_i\nabla_{\vtheta}\gL(f(\vx_i, \vtheta_0), y_i))\right| \\
    &\leq \frac{1}{m}\|\vtheta_0\|_2\sum_{i=1}^m\left\|\rmH_i\nabla_{\vtheta}\gL(f(\vx_i, \vtheta_0), y_i)\right\|_2 \\ 
    &\leq \frac{1}{m}\|\vtheta_0\|_2\sum_{i=1}^m\|\rmH_i\|_2\|\nabla_{\vtheta}\gL(f(\vx_i, \vtheta_0), y_i)\|_2 \\ 
    &\leq (\gamma\rho_1^2 + \beta \rho_2)\sqrt{d + 2\sqrt{d\ln\frac{1}{\delta'}}+2\ln\frac{1}{\delta'}} \cdot \frac{1}{m}\sum_{i=1}^m \left\|\nabla_{\vtheta}\ell(f(\vx_i, \vtheta_0), y_i)\right\|_2 \\
    &\leq \beta(\gamma\rho_1^2 + \beta \rho_2)\sqrt{d + 2\sqrt{d\ln\frac{1}{\delta'}}+2\ln\frac{1}{\delta'}} \gM_{\tracenorm} \ . \label{grasp&grad}
\end{aligned}
\end{equation}

Similarly, let $\delta=(m+1)\delta'$ and $C_3 = \beta(\gamma\rho_1^2 + \beta \rho_2)\sqrt{d + 2\sqrt{d\ln\frac{m+1}{\delta}}+2\ln\frac{m+1}{\delta}}$, with a high probability (i.e., at least $1-\delta$), we finally have
\begin{equation}
    \gM_{\grasp} \leq C_3\gM_{\tracenorm} \ ,
\end{equation}
which concludes our proof.

\begin{remark}
\emph{In addition to the provable theoretical connection between $\gM_{\tracenorm}$ and other training-free metrics from Sec.~\ref{sec:metrics}, we can further reveal the connection between $\gM_{\tracenorm}$ and recently proposed training-free metric $\gM_{\knas}$ in \citep{knas}. Specifically, let the training-free metric $\gM_{\knas}$ be defined as
\begin{equation}
    \gM_{\knas} \triangleq \sqrt{\left|\frac{1}{m^2}\sum_{i,j=1}^m \nabla_{\vtheta}f(\vx_i, \vtheta_0)^{\top}\nabla_{\vtheta}f(\vx_j, \vtheta_0)\right|} \ .
\end{equation}
Of note, we have adapted the original $\knas$ metric in \citep{knas} to match the mathematical form of other training-free metrics in Sec.~\ref{sec:metrics}. Interestingly, training-free metric $\gM_{\knas}$ is also gradient-based. As a result, we can also theoretically connect $\gM_{\knas}$ with $\gM_{\tracenorm}$ in a similar way:
\begin{equation}
\begin{aligned}
\gM_{\knas}^2 &= \left|\frac{1}{m^2}\sum_{i,j=1}^m \nabla_{\vtheta}f(\vx_i, \vtheta_0)^{\top}\nabla_{\vtheta}f(\vx_j, \vtheta_0)\right| \\
&\leq \frac{1}{m^2} \sqrt{m^2\sum_{i,j=1}^m \left(\nabla_{\vtheta}f(\vx_i, \vtheta_0)^{\top}\nabla_{\vtheta}f(\vx_j, \vtheta_0)\right)^2} \\
&= \frac{1}{m} \left\|\rmTheta_0\right\|_{\frob} \leq \frac{1}{m} \left\|\rmTheta_0\right\|_{\tr} = \gM_{\tracenorm}^2
\end{aligned}
\end{equation}
where the first inequality follows from the Cauchy-Schwarz inequality and the second equality is based on the definition of Frobenius norm. The last inequality derives from the matrix inequality $\|\cdot\|_{\frob} \leq \|\cdot\|_{\tr}$ while the last equality is obtained based on the definition of $\gM_{\tracenorm}$. Therefore, we have the following theoretical connection between $\gM_{\knas}$ and $\gM_{\tracenorm}$, which we will validate empirically in Appendix~\ref{sec:app:connection}.
\begin{equation}
    \gM_{\knas} \leq \gM_{\tracenorm} \ .
\end{equation}
Consequently, the theoretical results and the HNAS framework in this paper are also applicable to the training-free metric $\gM_{\knas}$. We have validated part of them empirically in Appendix \ref{sec:app:empirical}.
}
\end{remark}

\begin{remark}
\emph{
Note that our assumptions about the Lipschitz continuity and the Lipschitz smoothness of loss function $\ell(\cdot, \cdot)$ are usually satisfied for commonly employed loss functions in practice, e.g., Cross Entropy and Mean Square Error. For example, \citet{nasi} have justified that these two commonly applied loss functions indeed satisfy the Lipschitz continuity assumption. As for their Lipschitz smoothness, following a similar analysis in \citep{nasi}, we can also verify that there exists a constant $c>0$ such that $\|\nabla_{f}^2\ell(f, \cdot)\|_2 \leq c$ for both Cross Entropy and Mean Square Error.
}
\end{remark}

\subsection{Proof of Theorem~\ref{th:generalization}}\label{sec:proof:th:generalization}
\subsubsection{Estimating the Rademacher Complexity of DNNs}
Note that the Rademacher complexity of a hypothesis class $\gG$ over dataset $S=\{(\vx_i, y_i)\}_{i=1}^m$ of size $m$ is usually defined as
\begin{equation}
    \gR_S(\gG) = \E_{\vepsilon \in \{\pm1\}^m}\left[\sup_{g \in \gG} \frac{1}{m}\sum_{i=1}^m \epsilon_i g(\vx_i)\right] \ ,
\end{equation}
with $\eps_i \in \{\pm 1\}$. Let $\vtheta_0$ be the initialized parameters of DNN model $f$, we then define the following hypotheses that will be used to prove our lemmas and theorems:
\begin{equation}
\begin{aligned}
    \gF \triangleq \{\vx \mapsto f(\vx, \vtheta_t): t > 0\}, \quad \gF^{\lin} \triangleq \{\vx \mapsto f(\vx, \vtheta_0) + \nabla_{\vtheta}f(\vx, \vtheta_0)^{\top} (\vtheta_t - \vtheta_0): t > 0\} 
\end{aligned}
\end{equation}
where $f_t \in \gF$ and $f_t^{\lin} \in \gF^{\lin}$ are the function determined by the DNN model $f$ and its corresponding linearization at step $t$ of their optimization, respectively. Of note, the $\vtheta_t$ in $f_t$ and $f_t^{\lin}$ are not identical and should instead be determined by the optimization of $f_t$ and $f_t^{\lin}$ independently. Interestingly, $f_t$ can then be well characterized by $f_t^{\lin}$ as proved in the following lemma.

\begin{lemma}[Theorem H.1 \citep{ntk-linear}]\label{th:ntk-linear}
    Let $n_1=\cdots=n_{L-1}=n$ and assume $\lambda_{\min}(\rmTheta_{\infty}) > 0$. There exist the constant $c>0$ and $N>0$ such that for any $n>N$ and any $\vx \in \sR^{n_0}$ with $\|\vx\|_2 \leq 1$, the following holds with probability at least $1-\delta$ over random initialization when applying gradient descent with learning rate $\eta < \eta_0$,
    \begin{equation*}
    \sup_{t\geq0}\left\|f_t-f^{\text{\normalfont lin}}_t\right\|_2 \leq \frac{c}{\sqrt{n}} \ .
    \end{equation*}
\end{lemma}

\begin{remark}
\emph{
According to \citep{ntk-linear}, $\lambda_{\min}(\rmTheta_{\infty}) > 0$ usually holds especially when any input $\vx$ from dataset $S$ satisfies $\left\|\vx\right\|_2=1$. In practice, $\left\|\vx\right\|_2=1$ can be achieved by normalizing each input $\vx$ from real-world dataset using its norm $\|\vx\|_2$, which typically servers as the data preprocessing procedure for the model training of DNNs.
}
\end{remark}

Moreover, we will show that the Rademacher complexity of the DNN model during model training (i.e., $\gF$) can also be bounded using its linearization model (i.e., $\gF^{\lin}$) based on the following lemmas.

\begin{lemma}\label{th:bounded-complexity}
    With Lemma~\ref{th:ntk-linear}, there exists a constant $c>0$ such that with probability at least $1-\delta$ over random initialization, the following holds
    \begin{equation*}
    \gR_S(\gF) \leq \gR_S(\gF^{\normalfont\lin}) + \frac{c}{\sqrt{n}} \ .
    \end{equation*}
\end{lemma}
\begin{proof}
Based on Lemma~\ref{th:ntk-linear}, given $\eps_i \in \{\pm 1\}$, with probability at least $1-\delta$, there exists a constant $c > 0$ such that 
\begin{equation}
    \eps_i f_t \leq \eps_i f^{\lin}_t + \frac{c}{\sqrt{n}} \ .
\end{equation}

Following the definition of Rademacher complexity, we can bound the complexity of $\gF$ by
\begin{equation}
\begin{aligned}
    \gR_S(\gF) &= \E_{\vepsilon \in \{\pm1\}^m}\left[\sup_{f \in \gF} \frac{1}{m}\sum_{i=1}^m \epsilon_i f(\vx_i, \vtheta)\right] \\
    &\leq \E_{\vepsilon \in \{\pm 1\}^m}\left[\sup_{f^{\lin} \in \gF^{\lin}} \frac{1}{m}\sum_{i=1}^m \left(\epsilon_i f^{\lin}(\vx_i) + \frac{c}{\sqrt{n}}\right) \right] \\
    &\leq \E_{\vepsilon \in \{\pm 1\}^m}\left[\sup_{f^{\lin} \in \gF^{\lin}} \frac{1}{m}\sum_{i=1}^m \epsilon_i f^{\lin}(\vx_i)\right] + \E_{\vepsilon \in \{\pm 1\}^m}\left[\frac{c}{\sqrt{n}} \right] \\
    & \leq \gR_S(\gF^{\lin}) + \frac{c}{\sqrt{n}} \ ,
\end{aligned}
\end{equation}
which completes the proof.
\end{proof}

\begin{lemma}\label{th:converged-parameter}
    Let $f(\rmX, \vtheta_0)\triangleq [f(\vx_1, \vtheta_0) \cdots f(\vx_m, \vtheta_0)]^{\top}$ and $\vy \triangleq [y_1 \cdots y_m]^{\top}$ be the outputs of DNN model $f$ at initialization and the target labels of a dataset $S=\{(\vx_i, y_i)\}_{i=1}^m$, respectively. Given MSE loss $\Ls = \sum_{i=1}^m \|f^{\normalfont\lin}(\vx_i, \vtheta) - y_i\|_2^2 / (2m)$ and NTK matrix at initialization $\rmTheta_0 = \nabla_{\vtheta}f(\rmX,\vtheta_0)\nabla_{\vtheta}f(\rmX,\vtheta_0)^{\top}$, assume $\lambda_{\min}(\rmTheta_0) > 0$, for any $t>0$, the following holds when applying gradient descent on $f^{\normalfont\lin}(\vx, \vtheta)$ with learning rate $\eta < m/\lambda_{\max}(\rmTheta_0)$:
    \begin{equation*}
        \|\vtheta_t - \vtheta_0\|_2 \leq \|\vtheta_{\infty} - \vtheta_0\|_2 =  \sqrt{\widehat{\vy}^{\top}\rmTheta_0^{-1}\widehat{\vy}}
    \end{equation*}
    where $\vtheta_t$ denotes the parameters of $f^{\normalfont\lin}$ at step $t$ of its model training and $\widehat{\vy} \triangleq \vy - f(\rmX, \vtheta_0)$. Besides, $\lambda_{\max}(\rmTheta_0)$ and $\lambda_{\min}(\rmTheta_0)$ denote the maximum and minimum eigenvalue of matrix $\rmTheta_0$.
\end{lemma}
\begin{proof}

Following the update of gradient descent on MSE with learning rate $\eta < m/\lambda_{\max}(\rmTheta_0)$, we have
\begin{equation}
\begin{aligned}
    \vtheta_{t+1} &= \vtheta_t - \frac{\eta}{m} \nabla_{\vtheta}f(\rmX, \vtheta_0)^{\top} \left(f^{\lin}(\rmX, \vtheta_t) - \vy\right) \ . \label{eq:sg-update}
\end{aligned}
\end{equation}
Note that $\nabla_{\vtheta}f(\rmX, \vtheta_0)$ is a $m \times d$ matrix and $f(\rmX, \vtheta_0), f^{\lin}(\rmX, \vtheta_0), \vy$ are $m$-dimensional column vectors. By subtracting $\vtheta_0$, multiplying $\nabla_{\vtheta}f(\rmX, \vtheta_0)$ and then adding $f(\rmX, \vtheta_0)$ on both sides of the equality above, we achieve
\begin{equation}
\begin{aligned}
    f(\rmX, \vtheta_0) + \nabla_{\vtheta}f(\rmX, \vtheta_0)(\vtheta_{t+1} - \vtheta_0) &= f(\rmX, \vtheta_0) + \nabla_{\vtheta}f(\rmX, \vtheta_0)(\vtheta_t - \vtheta_0) - \frac{\eta}{m} \rmTheta_0 \left(f^{\lin}(\rmX, \vtheta_t) - \vy\right) \ ,
\end{aligned}
\end{equation}
which can be simplified as
\begin{equation}
\begin{aligned}
    f^{\lin}(\rmX, \vtheta_{t+1}) &= f^{\lin}(\rmX, \vtheta_t) - \frac{\eta}{m}\vTheta_0\Big[f^{\lin}(\rmX, \vtheta_t) - \vy\Big] \\
    &= \left(\rmI - \frac{\eta}{m} \rmTheta_0\right) f^{\lin}(\rmX, \vtheta_t) + \frac{\eta}{m} \rmTheta_0\vy \ .
\end{aligned}
\end{equation}

By recursively applying the equality above for $t+1$ times, we finally achieve
\begin{equation}
\begin{aligned}
    f^{\lin}(\rmX, \vtheta_{t+1}) &\stackrel{(a)}{=} \left(\rmI -\frac{\eta}{m} \rmTheta_0 \right)^{t+1} f^{\lin}(\rmX, \vtheta_0) + \sum_{j=0}^t \left(\rmI -\frac{\eta}{m} \rmTheta_0 \right)^{j} \left(\frac{\eta}{m} \rmTheta_0\vy\right) \\
    &\stackrel{(b)}{=} \left(\rmI -\frac{\eta}{m} \rmTheta_0 \right)^{t+1} f(\rmX, \vtheta_0) + \left[\rmI - (\rmI -\frac{\eta}{m} \rmTheta_0)^{t+1}\right]\left(\frac{\eta}{m} \rmTheta_0\right)^{-1}\frac{\eta}{m}\rmTheta_0\vy \\[8.2pt]
    &\stackrel{(c)}{=} \left(\rmI -\frac{\eta}{m} \rmTheta_0 \right)^{t+1} \Big(f(\rmX, \vtheta_0) - \vy\Big) + \vy 
\end{aligned}
\end{equation}
where $(b)$ follows from the sum of geometric series for matrix with $\eta < m/\lambda_{\max}(\rmTheta_0)$ as well as the fact that $f^{\lin}(\rmX, \vtheta_0)=f(\rmX, \vtheta_0)$. Note that this result can be integrated into \eqref{eq:sg-update} and provide the following explicit form of $\vtheta_{t+1} - \vtheta_0$ after applying gradient descent for $t+1$ times:
\begin{equation}
\begin{aligned}
    \vtheta_{t+1} - \vtheta_0 &= \sum_{k=0}^{t} \vtheta_{k+1} - \vtheta_{k} \\
    &= \frac{\eta}{m} \nabla_{\vtheta}f(\rmX, \vtheta_0)^{\top} \sum_{k=0}^t \left(\rmI -\frac{\eta}{m} \rmTheta_0 \right)^{k} \left(\vy - f(\rmX, \vtheta_0)\right) \\
    &= \frac{\eta}{m} \nabla_{\vtheta}f(\rmX, \vtheta_0)^{\top} \sum_{k=0}^{t} (\rmI - \frac{\eta}{m} \rmTheta_0)^{k} \widehat{\vy}
\end{aligned}
\end{equation}

Since $\rmTheta_0$ is symmetric, we can alternatively represent $\rmTheta_0$ as $\rmTheta_0 = \rmV \rmLambda \rmV^{\top}$ using principal component analysis (PCA) where $\rmV$ and $\rmLambda$ denotes the matrix of eigenvectors $\{\vv_i\}_{i=1}^m$ and eigenvalues $\{\lambda_i\}_{i=1}^m$, respectively. Based on this representation, we have
\begin{equation}
\begin{aligned}
    \|\vtheta_{t+1} - \vtheta_0\|_2 &= \sqrt{(\vtheta_{t+1} - \vtheta_0)^{\top}(\vtheta_{t+1} - \vtheta_0)} \\
    & = \frac{\eta}{m}\sqrt{\widehat{\vy}^{\top} \sum_{k=0}^{t} (\rmI - \frac{\eta}{m} \rmTheta_0)^{k} \nabla_{\vtheta}f(\rmX, \vtheta_0) \nabla_{\vtheta}f(\rmX, \vtheta_0)^{\top} \sum_{k'=0}^{t} (\rmI - \frac{\eta}{m} \rmTheta_0)^{k'} \widehat{\vy}} \\
    & = \frac{\eta}{m}\sqrt{\widehat{\vy}^{\top} \sum_{k=0}^{t} (\rmI - \frac{\eta}{m} \rmTheta_0)^{k} \rmTheta_0 \sum_{k'=0}^{t} (\rmI - \frac{\eta}{m} \rmTheta_0)^{k'} \widehat{\vy}} \\
    & = \frac{\eta}{m}\sqrt{\widehat{\vy}^{\top} \sum_{k=0}^{t} (\rmI - \frac{\eta}{m} \rmV \rmLambda \rmV^{\top})^{k} \rmV \rmLambda \rmV^{\top} \sum_{k'=0}^{t} (\rmI - \frac{\eta}{m} \rmV \rmLambda \rmV^{\top})^{k'} \widehat{\vy}} \\
    & = \frac{\eta}{m}\sqrt{\widehat{\vy}^{\top} \rmV \sum_{k=0}^{t} (\rmI - \frac{\eta}{m} \rmLambda )^{k} \rmV^{\top} \rmV \rmLambda \rmV^{\top} \rmV \sum_{k'=0}^{t} (\rmI - \frac{\eta}{m} \rmLambda)^{k'} \rmV^{\top} \widehat{\vy}} \\
    & = \frac{\eta}{m}\sqrt{\widehat{\vy}^{\top} \rmV \sum_{k=0}^{t} (\rmI - \frac{\eta}{m} \rmLambda )^{k} \rmLambda \sum_{k'=0}^{t} (\rmI - \frac{\eta}{m} \rmLambda)^{k'} \rmV^{\top} \widehat{\vy}} \\
    & = \frac{\eta}{m}\sqrt{\sum_{i=1}^m \lambda_i\left[\sum_{k=0}^t(1 - \frac{\eta}{m}\lambda_i)^{k}\right]^2 (\vv_i^{\top}\widehat{\vy})^2} \ . \label{eq:param-norm}
\end{aligned}
\end{equation}

Since $\eta < m/\lambda_{\max}(\rmTheta_0)$ and $\lambda_{\min}(\rmTheta_0) > 0$, we have $0 < 1 - \eta\lambda_i/m < 1$ and hence

\begin{equation}
\begin{aligned}
     \|\vtheta_{t} - \vtheta_0\|_2 &= \frac{\eta}{m}\sqrt{\sum_{i=1}^m \lambda_i\left[\sum_{k=0}^{t-1}(1 - \frac{\eta}{m}\lambda_i)^{k}\right]^2 (\vv_i^{\top}\widehat{\vy})^2} \\
     &\leq \frac{\eta}{m}\sqrt{\sum_{i=1}^m \lambda_i\left[\sum_{k=0}^{t}(1 - \frac{\eta}{m}\lambda_i)^{k}\right]^2 (\vv_i^{\top}\widehat{\vy})^2} \\[10.2pt]
     &= \|\vtheta_{t+1} - \vtheta_0\|_2 
\end{aligned}
\end{equation}

We complete the proof by recursively applying the inequalities above
\begin{equation}
\begin{aligned}
    \|\vtheta_{t} - \vtheta_0\|_2 &\leq \|\vtheta_{\infty} - \vtheta_0\|_2 \\
    &= \frac{\eta}{m}\sqrt{\sum_{i=1}^m \lambda_i\left[\sum_{k=0}^{\infty}(1 - \frac{\eta}{m}\lambda_i)^{k}\right]^2 (\vv_i^{\top}\widehat{\vy})^2} \\
    &= \frac{\eta}{m}\sqrt{\sum_{i=1}^m \lambda_i\left[ \frac{1}{\eta\lambda_i/m} \right]^2 (\vv_i^{\top}\widehat{\vy})^2} \\
    &= \sqrt{\sum_{i=1}^m \lambda_i^{-1} (\vv_i^{\top}\widehat{\vy})^2} \\
    &= \sqrt{\widehat{\vy}^{\top}\rmTheta_0^{-1}\widehat{\vy}}
\end{aligned}
\end{equation}
\end{proof}

\begin{lemma}[\citet{rademacher-linear}]\label{th:rademacher-linear}
Let $\gG \triangleq \{\vx \mapsto \vw^T\vx: \|\vw\|_2 \leq R\}$ be a family of linear functions defined over $\sR^d$ with bounded weight. Then the empirical Rademacher complexity of $\gG$ for $m$ samples $S \triangleq (\vx_1, \cdots, \vx_m)$ admits the following upper bounds:
\begin{equation*}
    \gR_S(\gG) \leq \frac{R}{m} \|\rmX^{\top}\|_{2,2} 
\end{equation*}
where $\rmX$ is the $d\times m$-matrix with $\vx_i$s as columns: $\rmX \triangleq [\vx_1 \cdots \vx_m]$.
\end{lemma}

Based on our Lemma \ref{th:bounded-complexity} and Lemma \ref{th:converged-parameter}, we can finally bound the Rademacher complexity of a DNN model during its model training (i.e., $\gF$) using its linearization model (i.e., $\gF^{\lin}$). Specifically, under the conditions in Theorem \ref{th:ntk-linear} and Lemma~\ref{th:converged-parameter}, there exist the constant $c>0$ and $N>0$ such that for any $n>N$, with probability at least $1 - \delta$ over initialization, we have
\begin{equation}
\begin{aligned}
    \gR_S(\gF) &\stackrel{(a)}{\leq} \gR_S(\gF^{\lin}) + \frac{c}{\sqrt{n}} \\
    &\stackrel{(b)}{=} \E_{\vepsilon \in \{\pm1\}^m}\left[\sup_{t \ge 0} \frac{1}{m} \sum_{i=1}^m \epsilon_i \left(f(\vx_i, \vtheta_0) + \nabla_{\vtheta}f(\vx_i, \vtheta_0)^{\top} (\vtheta_t - \vtheta_0)\right) \right] + \frac{c}{\sqrt{n}} \\
    &\stackrel{(c)}{=} \E_{\vepsilon \in \{\pm1\}^m}\left[\sup_{t \ge 0}\frac{1}{m}\sum_{i=1}^m \epsilon_i \nabla_{\vtheta}f(\vx_i, \vtheta_0)^{\top} (\vtheta_t - \vtheta_0)\right] + \frac{1}{m}\sum_{i=1}^m \E_{\vepsilon \in \{\pm1\}^m}\left[\eps_i\right] f(\vx_i, \vtheta_0) + \frac{c}{\sqrt{n}} \\[5.2pt]
    &\stackrel{(d)}{\leq} \frac{\|\vtheta_{\infty} - \vtheta_0\|_2 \|\nabla_{\vtheta}f(\rmX, \vtheta_0)\|_{2,2}}{m} + \frac{c}{\sqrt{n}} \\[5.2pt]
    &\stackrel{(e)}{\leq} \frac{\|\nabla_{\vtheta}f(\rmX, \vtheta_0)\|_{2,2}}{m} \sqrt{\widehat{\vy}^{\top}\rmTheta_0^{-1}\widehat{\vy}} + \frac{c}{\sqrt{n}} \\
    &\stackrel{(f)}{\leq} \sqrt{\kappa\lambda_0} \cdot \sqrt{\frac{\widehat{\vy}^{\top}\rmTheta_0^{-1}\widehat{\vy}}{m}} + \frac{c}{\sqrt{n}} \label{eq:dnn-complexity}
\end{aligned}
\end{equation}
where $(d)$ derives from Lemma~\ref{th:rademacher-linear} and $(f)$ derives from the following inequalities based on the definition $\kappa \triangleq \lambda_{\max}(\rmTheta_0) / \lambda_{\min}(\rmTheta_0)$ and $\lambda_0 \triangleq \lambda_{\min}(\rmTheta_0)$.
\begin{equation}
\begin{aligned}
    \|\nabla_{\vtheta}f(\rmX, \vtheta_0)\|_{2,2} &= \sqrt{\sum_{i=1}^m \|\nabla_{\vtheta}f(\vx_i, \vtheta_0)\|_2^2} \\
    &=\sqrt{\sum_{i=1}^m \lambda_i(\rmTheta_0)} \\[8.2pt]
    &\leq \sqrt{m\kappa\lambda_0} \ .
\end{aligned}
\end{equation}

\subsubsection{Deriving the Generalization Bound for DNNs using Training-free Metrics}
Define the generalization error on the data distribution $\gD$ as $\Ls_{\gD}(g) \triangleq \E_{(\vx,y) \sim \gD}\ell(g(\vx), y)$ and the empirical error on the dataset $S=\{(\vx_i, y_i)\}_{i=1}^m$ that is randomly sampled from $\gD$ as $\Ls_S(g) \triangleq \sum_{i=1}^m\ell(g(\vx_i), y_i)$. Given the loss function $\ell(\cdot, \cdot)$ and the Rademacher complexity of any hypothesis class $\gG$, the generalization error on the hypothesis class $\gG$ can then be estimated by the empirical error using the following lemma.
\begin{lemma}[\citet{foundation}]\label{th:rademacher-generalization}
Suppose the loss function $\ell(\cdot, \cdot)$ is bounded in $[0, 1]$ and is $\beta$-Lipschitz continuous in the first argument. Then with probability at least $1-\delta$ over dataset $S$ of size $m$:
\begin{equation*}
    \sup_{g \in \gG}\{\Ls_{\gD}(g) - \Ls_S(g)\} \leq 2\beta\gR_S(\gG) + 3\sqrt{\ln(2/\delta)/(2m)} \ .
\end{equation*}
\end{lemma}

\begin{lemma}\label{th:inverse-trace}
    For a symmetric matrix $\rmA \in \sR^{m \times m}$ with eigenvalues $\{\lambda_i\}_{i=1}^m$ in an ascending order, define $\kappa \triangleq \lambda_m / \lambda_1$, the following inequality holds if $\lambda_1 > 0$,
    \begin{equation*}
        \left\|\rmA\right\|_{\normalfont\tr}\left\|\rmA^{-1}\right\|_{\normalfont\tr} \leq m^2\kappa \ .
    \end{equation*}
\end{lemma}
\begin{proof}
Since eigenvalues $\{\lambda_i\}_{i=1}^m$ are in an ascending order, we have
\begin{align}
    \frac{\lambda_m}{\kappa} \leq \lambda_i \leq \lambda_1\kappa \ .
\end{align}
Based on the results above, we can connect the matrix norm $\|\rmA\|_{\tr}$ and $\|\rmA^{-1}\|_{\tr}$ with
\begin{equation}
\begin{aligned}
    \left\|\rmA\right\|_{\tr}\left\|\rmA^{-1}\right\|_{\tr} = (\sum_{i=1}^m \lambda_i) \cdot (\sum_{i=1}^m \lambda_i^{-1}) \leq \left(m\lambda_1\kappa\right)\cdot\frac{m\kappa}{\lambda_m} = \frac{m^2\kappa^2}{\kappa} = m^2\kappa \ ,
\end{aligned}
\end{equation}
which concludes the proof.
\end{proof}

We are now able to prove Theorem \ref{th:generalization} by combining the results in Lemma~\ref{th:rademacher-generalization} and \eqref{eq:dnn-complexity}. Specifically, under the conditions in Theorem \ref{th:ntk-linear} and Lemma~\ref{th:converged-parameter}, there exist constant $c, N>0$ such that for any $f_t \in \gF$ and any $n > N$, the following holds with probability at least $1-2\delta$ over random initialization,
\begin{equation}
\begin{aligned}
    \Ls_{\gD}(f_t) &\leq \Ls_S(f_t) + 2\beta\gR_S(\gF) + 3\sqrt{\frac{\ln(2/\delta)}{2m}} \\
    &\leq \Ls_S(f_t) + 2\beta\sqrt{\kappa\lambda_0} \cdot \sqrt{\frac{\widehat{\vy}^{\top}\rmTheta_0^{-1}\widehat{\vy}}{m}} + \frac{2\beta c}{\sqrt{n}} + 3\sqrt{\frac{\ln(2/\delta)}{2m}} \ . \label{eq:generalization-bound-ntk}
\end{aligned}
\end{equation}

Assume $f(\vx, \vtheta_0)$ and $y$ are bounded in $[0,1]$ for any pair $(\vx, y)$ in the dataset $S$, let $\{\vv_i\}_{i=1}^m$ and $\{\lambda_i\}_{i=1}^m$ be the eigenvectors and eigenvalues of $\rmTheta_0$, respectively, we then have $\widehat{y} \in [-1,1]^m$ and the following inequalities:
\begin{equation}
\begin{aligned}
    \widehat{\vy}^{\top}\rmTheta_0^{-1}\widehat{\vy} &= \sum_{i=1}^m \frac{(\vv_i^{\top}\widehat{\vy})^2}{\lambda_i} \leq \sum_{i=1}^m \frac{\|\vv_i\|_2^2\|\widehat{\vy}\|_2^2}{\lambda_i} \leq \sum_{i=1}^m \frac{m}{\lambda_i} \ .
\end{aligned}
\end{equation}

Based on the fact that $\|\rmTheta_0\|_{\tr} = \sum_{i=1}^m \lambda_i$ and Lemma~\ref{th:inverse-trace}, we finally achieve
\begin{equation}
\begin{aligned}
    \sqrt{\frac{\widehat{\vy}^{\top}\rmTheta_0^{-1}\widehat{\vy}}{m}} &\leq \sqrt{\left\|\rmTheta_0^{-1}\right\|_{\tr}} \leq \frac{m\sqrt{\kappa}}{\sqrt{\|\rmTheta_0\|_{\tr}}} = \frac{\sqrt{m\kappa}}{\gM_{\tracenorm}} \ . \label{eq:bounded-complexity}
\end{aligned}
\end{equation}

By introducing \eqref{eq:bounded-complexity} into \eqref{eq:generalization-bound-ntk}, with $\lambda_0 \leq 1$, we have
\begin{equation}
\begin{aligned}
    \Ls_{\gD}(f_t) 
    &\leq \Ls_S(f_t) + \frac{2\beta\kappa\sqrt{m}}{\gM_{\tracenorm}} + \frac{2\beta c}{\sqrt{n}} + 3\sqrt{\frac{\ln(2/\delta)}{2m}} \ . \label{eq:bound-tracenorm}
\end{aligned}
\end{equation}

Let $\gM$ be any metric introduced in Sec.~\ref{sec:metrics}, based on the results in our Theorem~\ref{th:connection} and the definition of $\gO(\cdot)$, the following inequality then holds with a high probability using the result above:
\begin{equation}
\begin{aligned}
    \Ls_{\gD}(f_t) \leq \Ls_{S}(f_t) + \gO(\kappa /\gM) \ , \label{eq:generalization-trace}
\end{aligned}
\end{equation}
which finally concludes our proof of Theorem~\ref{th:generalization}.

\begin{remark}
\emph{
Our \eqref{eq:generalization-trace} still holds when $\lambda_0 \leq z (z \neq 1)$, i.e., by simply placing $z$ into our \eqref{eq:bound-tracenorm}.
Though our conclusion is based on the initialization using standard normal distribution and over-parameterized DNNs, our empirical results in Appendix~\ref{sec:app:ablation} show that this conclusion can also hold for DNNs initialized using other methods and also DNNs of small layer width.
}
\end{remark}

\subsection{Proof of Corollary~\ref{corol:generalization-non-realizable}}\label{sec:proof:corol:generalization-non-realizable}
To prove our Corollary~\ref{corol:generalization-non-realizable}, we firstly consider the convergence of $f_t^{\lin}$ under the same conditions in Theorem~\ref{th:generalization}. Specifically, following the notations and results in Lemma~\ref{th:converged-parameter}, let $\{\vv_i\}_{i=1}^m$ and $\{\lambda_i\}_{i=1}^m$ be the eigenvectors and eigenvalues of $\rmTheta_0$, respectively, we have
\begin{equation}
\begin{aligned}
    \Ls_S(f_t^{\lin}) &\stackrel{(a)}{=} \frac{1}{2m}\left\|f^{\lin}(\rmX, \vtheta_t) - \vy\right\|^2_2 \\
    &\stackrel{(b)}{=} \frac{1}{2m}\left\|\left(\rmI - \frac{\eta}{m}\rmTheta_0\right)^{t}\left(f(\rmX, \vtheta_0) - \vy\right)\right\|^2_2 \\
    &\stackrel{(c)}{=} \frac{1}{2m}\left\|\left(\rmI - \frac{\eta}{m}\rmTheta_0\right)^{t}\widehat{\vy}\right\|^2_2 \\
    &\stackrel{(d)}{=} \frac{1}{2m}\sum_{i=1}^m\left(1-\frac{\eta}{m}\lambda_i\right)^{2t}\left(\vv_i^{\top}\widehat{\vy}\right)^2 \\
    &\stackrel{(e)}{\leq} \frac{1}{2m}\sum_{i=1}^m\left(1-\frac{\eta}{m}\lambda_i\right)^{2t}\left\|\vv_i\right\|_2^2\left\|\widehat{\vy}\right\|_2^2 
\end{aligned}
\end{equation}
where $(d)$ follows the same derivation in \eqref{eq:param-norm}. Moreover, based on $\widehat{\vy} \in [-1,1]^m$ and the fact that $\|\vv_i\|_2 = 1$, for any $t>0$ (i.e., $t= 1,2,\cdots$), we naturally have
\begin{equation}
\begin{aligned}
    \Ls_S(f_t^{\lin}) &\stackrel{(a)}{\leq} \frac{1}{2}\sum_{i=1}^m \left(1-\frac{\eta}{m}\lambda_i\right)^{2t} \\
    &\stackrel{(b)}{\leq} \frac{1}{2}\left(\sum_{i=1}^m 1-\frac{\eta}{m}\lambda_i\right)^{2t} \\[3pt]
    &\stackrel{(c)}{=} \frac{1}{2}\left(m - \frac{\eta}{m} \|\rmTheta_0\|_{\tr}\right)^{2t} \\[3pt]
    &\stackrel{(d)}{=} \frac{1}{2}\left(m - \eta \gM^2_{\tracenorm}\right)^{2t} \\[3pt]
    &\stackrel{(e)}{\leq} \frac{1}{2}\left(m - \eta\gM^2 / C\right)^{2t} 
\end{aligned}
\end{equation}
where 
$(e)$ is based on the results in our Theorem \ref{th:connection}: For any training-free metric $\gM$ introduced in Sec.~\ref{sec:metrics}, there exists a constant $C$ such that the following holds with a high probability,
\begin{equation}
\begin{aligned}
\gM^2 \leq C\gM^2_{\tracenorm} \quad \Rightarrow \quad m - \eta\gM^2/C \geq m - \eta \gM^2_{\tracenorm} \ .
\end{aligned}
\end{equation}

Based on Lemma~\ref{th:ntk-linear} and the fact that loss function $\ell(f,y)=(f-y)^2/2$ is 1-Lipschitz continuous in the first argument, the following then holds with a high probability
\begin{equation}
\begin{aligned}
\left|\Ls_S(f_t) - \Ls_S(f_t^{\lin})\right| \leq \left|f_t - f_t^{\lin}\right| \leq \gO(\frac{1}{\sqrt{n}}) \ .
\end{aligned}
\end{equation}

By introducing the results above into our Theorem~\ref{th:generalization} with $1/\sqrt{n}$ being absorbed in $\gO(\cdot)$, we finally achieve the following results with a high probability,
\begin{equation}
\begin{aligned}
\Ls_{\gD}(f_t) &\leq \Ls_{S}(f_t) + \gO(\kappa/\gM) \leq \Ls_{S}(f_t^{\lin}) + \gO(\kappa/\gM) \\
&\leq \frac{1}{2}\left(m - \eta\gM^2 / C\right)^{2t} + \gO(\kappa/\gM) \ ,
\end{aligned}
\end{equation}
which thus concludes our proof.

\subsection{Proof of Theorem~\ref{th:topology}}\label{sec:proof:th:topology}

\begin{figure}[t]
\centering
\begin{subfigure}[t]{0.25\textwidth}
\centering
\includegraphics[width=\textwidth]{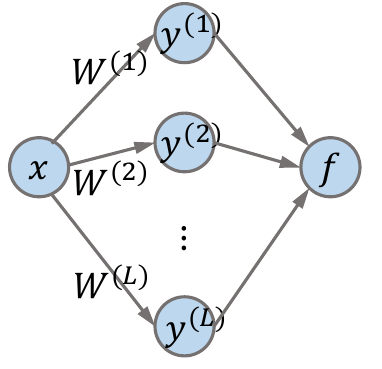}
\caption{Wide architecture}
\end{subfigure}
\hspace{8em}
\begin{subfigure}[t]{0.24\textwidth}
\centering
\includegraphics[width=\textwidth]{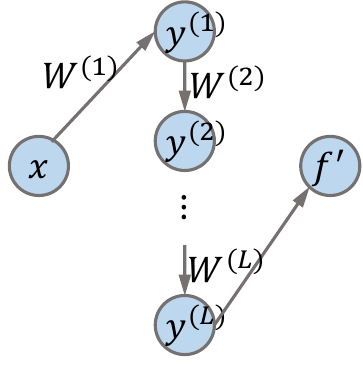}
\caption{Deep architecture}
\end{subfigure}
\caption{Two different architecture topologies for our analysis.}
\label{fig:topology}
\vskip -0.1in
\end{figure}

Let $\rmW^{(i)}_{j\cdot}$ denote the $j$-th row of matrix $\rmW^{(i)}$, based on the definition of $f$ and $f'$ in Sec.~\ref{sec:topology}, we can compute the gradient (represented as a column vector) of $\rmW^{(i)}_{j\cdot}$ for function $f$ and $f'$ respectively as below
\begin{equation}
\begin{aligned}
\nabla_{\rmW^{(i)}_{j\cdot}}f(\vx) &= \vx \\
\nabla_{\rmW^{(i)}_{j\cdot}}f'(\vx) &= \left(\prod_{k'=1}^{i-1}\rmW^{(k')}\vx\right) \vone^{\top}\left(\prod_{k=i+1}^L \rmW^{(k)}\right)_{\cdot j}  \label{eq:topology-ntk}
\end{aligned}
\end{equation}
where $\left(\prod_{k=i+1}^L \rmW^{(k)}\right)_{\cdot j}$ is defined as the $j$-th column of matrix $\left(\prod_{k=i+1}^L \rmW^{(k)}\right)$, i.e., 
\begin{equation}
\begin{aligned}
\left(\prod_{k=i+1}^L \rmW^{(k)}\right)_{\cdot j} \triangleq \left(\rmW^{(i+1)}\cdots\rmW^{(L)}\right)_{\cdot j} = \rmW^{(L)}\rmW^{(L-1)}\cdots\rmW^{(i+1)}_{\cdot j} \ ,
\end{aligned}
\end{equation}

Consequently, the NTK matrix of initialized wide architecture can be represented as
\begin{equation}
\begin{aligned}
\rmTheta_0(\vx, \vx') &= \sum_{i=1}^L \sum_{j=1}^{n} \left(\nabla_{\rmW^{(i)}_{j\cdot}}f(\vx)\right)^{\top}\nabla_{\rmW^{(i)}_{j\cdot}}f(\vx') \\
&= \sum_{i=1}^L \sum_{j=1}^{n}\vx^{\top}\vx' = nL\cdot\vx^{\top}\vx' \ .
\end{aligned}
\end{equation}

Meanwhile, the NTK matrix of initialized deep architecture can be represented as
\begin{equation}
\begin{aligned}
\rmTheta'_0(\vx, \vx') &=\sum_{i=1}^L \sum_{j=1}^{n} \left(\nabla_{\rmW^{(i)}_{j\cdot}}f'(\vx)\right)^{\top}\nabla_{\rmW^{(i)}_{j\cdot}}f'(\vx) \\
&=\sum_{i=1}^L \sum_{j=1}^{n} \left(\left(\prod_{k'=1}^{i-1}\rmW^{(k')}\vx\right) \vone^{\top}\left(\prod_{k=i+1}^L \rmW^{(k)}\right)_{\cdot j}\right)^{\top} \left(\prod_{k'=1}^{i-1}\rmW^{(k')}\vx'\right) \vone^{\top}\left(\prod_{k=i+1}^L \rmW^{(k)}\right)_{\cdot j} \\
&=\sum_{i=1}^L \sum_{j=1}^{n} \left(\vone^{\top}\left(\prod_{k=i+1}^L \rmW^{(k)}\right)_{\cdot j}\right)^2 \vx^{\top} \left(\prod_{k'=1}^{i-1}\rmW^{(k')}\right)^{\top} \left(\prod_{k'=1}^{i-1}\rmW^{(k')}\right)\vx' \\
&=\vx^{\top} \sum_{i=1}^L \sum_{j=1}^{n} \left(\vone^{\top}\left(\prod_{k=i+1}^L \rmW^{(k)}\right)_{\cdot j}\right)^2 \left(\prod_{k'=1}^{i-1}\rmW^{(k')}\right)^{\top} \left(\prod_{k'=1}^{i-1}\rmW^{(k')}\right)\vx' \ .
\end{aligned}
\end{equation}

Since each element in $\rmW^{(i)}$ is initialized using standard normal distribution, we have following simplified expectation by exploring the fact that $\E\left[\rmW^{(i)}\right] = \vzero\vzero^{\top}$ and $\E\left[\left(\rmW^{(i)}\right)^{\top}\rmW^{(i)}\right] = n\rmI$.
\begin{equation}
\begin{aligned}
\E\left[\left(\prod_{k'=1}^{i-1}\rmW^{(k')}\right)^{\top}\prod_{k'=1}^{i-1}\rmW^{(k')}\right] &= \E\left[\left(\rmW^{(1)}\right)^{\top}\cdots\left(\rmW^{(i-1)}\right)^{\top}\rmW^{(i-1)}\cdots\rmW^{(1)}\right] \\
&=\E\left[\left(\rmW^{(1)}\right)^{\top}\E\left[\cdots\E\left[\left(\rmW^{(i-1)}\right)^{\top}\rmW^{(i-1)}\right]\cdots\right]\rmW^{(1)}\right] \\
&=\E\left[\left(\rmW^{(1)}\right)^{\top}\E\left[\cdots\E\left[\left(\rmW^{(i-2)}\right)^{\top}\left(n\rmI\right)\rmW^{(i-2)}\right]\cdots\right]\rmW^{(1)}\right] \\[6.2pt]
&=n^{i-1}\rmI \ .
\end{aligned}
\end{equation}

Similarly, we also have
\begin{equation}
\begin{aligned}
\E\left[\left(\vone^{\top}\left(\prod_{k=i+1}^L \rmW^{(k)}\right)_{\cdot j}\right)^2\right] &=\vone^{\top} \E\left[\left(\prod_{k=i+1}^L \rmW^{(k)}\right)_{\cdot j}\left(\left(\prod_{k=i+1}^L \rmW^{(k)}\right)_{\cdot j}\right)^{\top}\right] \vone \\
&=\vone^{\top}\E\left[\rmW^{(L)} \E\left[\cdots\E\left[\rmW^{(i+1)}_{\cdot j}\left(\rmW^{(i+1)}_{\cdot j}\right)^{\top}\right]\cdots\right]\left(\rmW^{(L)}\right)^{\top}\right]\vone \\
&=\vone^{\top}\E\left[\rmW^{(L)} \E\left[\cdots \E\left[\rmW^{(i+2)}\rmI \left(\rmW^{(i+2)}\right)^{\top}\right] \cdots\right]\left(\rmW^{(L)}\right)^{\top}\right]\vone \\
&=\vone^{\top}\E\left[\rmW^{(L)} \E\left[\cdots \E\left[\rmW^{(i+3)}n\rmI \left(\rmW^{(i+3)}\right)^{\top}\right] \cdots\right]\left(\rmW^{(L)}\right)^{\top}\right]\vone \\
&= n^{l-i-1}\vone^{\top}\vone\\[8.2pt]
&= n^{L-i} \ .
\end{aligned}
\end{equation}

Since $\rmW^{(i)}$ in each layer is initialized independently, we achieve the following result by introducing the equality above and expectation over model parameters into \eqref{eq:topology-ntk}.
\begin{equation}
\begin{aligned}
\E\left[\rmTheta'_0(\vx,\vx')\right] &= \vx^{\top} \E\left[\sum_{i=1}^L \sum_{j=1}^{n} \left(\vone^{\top}\left(\prod_{k=i+1}^L \rmW^{(k)}\right)_{\cdot j}\right)^2 \left(\prod_{k'=1}^{i-1}\rmW^{(k')}\right)^{\top} \left(\prod_{k'=1}^{i-1}\rmW^{(k')}\right)\right]\vx' \\
&= \vx^{\top}\left(\sum_{i=1}^L \sum_{j=1}^{n} \E\left[\left(\vone^{\top}\left(\prod_{k=i+1}^L \rmW^{(k)}\right)_{\cdot j}\right)^2\right]\E\left[ \left(\prod_{k'=1}^{i-1}\rmW^{(k')}\right)^{\top}\prod_{k'=1}^{i-1}\rmW^{(k')}\right]\right) \vx' \\
&= \vx^{\top}\left(\sum_{i=1}^L \sum_{j=1}^{n} n^{L-i}\cdot n^{i-1}\rmI \right) \vx' \\[8.2pt]
&= Ln^L \vx^{\top}\vx' \ .
\end{aligned}
\end{equation}

By exploiting the fact that $\rmX^{\top}\rmX = \rmI$ with $\rmX \triangleq [\vx_1 \vx_2\cdots\vx_m]$, we finally conclude the proof by 
\begin{equation}
\begin{aligned}
    \rmTheta_0(\rmX, \rmX) &= Ln\cdot \rmI \\
    \E\left[\rmTheta'_0(\rmX, \rmX)\right] &= Ln^L  \cdot \rmI \ .
\end{aligned}
\end{equation}

\section{Optimization and Experimental Details}
\subsection{Optimization Details for Algorithm \ref{alg:hnas}}\label{sec:app:opt-details}

\paragraph{Solution to the Training-Free NAS Objective \eqref{eq:nas-final}.} Following the common practice in \citep{naswot, knas}, to solve \eqref{eq:nas-final} for the every iteration of our Algorithm \ref{alg:hnas} in practice, we independently and randomly sample a large pool of architectures from the search space to evaluate their training-free metrics and then select the architecture achieving the optimum value of   \eqref{eq:nas-final} (given the values of $\mu$ and $\nu$) from all sampled architectures. 
Meanwhile, following the common practice in \citep{zero-cost}, the training-free metrics of these sampled architectures are evaluated using a batch of sampled data as introduced in Sec.~\ref{sec:exp:connection}.

\paragraph{Introduction to the BO Applied in HNAS.} BO is a type of gradient-free optimization algorithm aiming to optimize a black-box or non-differentiable objective function by iteratively selecting an input (to only evaluate/query its function value) that intuitively trades off between sampling an input likely achieving optimum (i.e., exploitation) given the current belief of the function modeled by a Gaussian process (GP) vs. improving the GP belief over the entire input domain (i.e., exploration) to guarantee finding the global optimum, which recently has been widely extended to various real-world problem settings in order to achieve better optimization in practice \citep{sto-bnts, SebICML22, metaBO, phong-aaai-2021, phong-uai-2021, dai2020federated, dai2021differentially, bala20, sim2021collaborative, yehong2019, dai2019, dai2020, verma2022bayesian}. 
Since we adopt the non-differentiable validation performance (i.e., validation error) as the objective function to be optimized (over $\mu$ and $\nu$) in our Algorithm \ref{alg:hnas}, BO will naturally be a better choice to find the optimal $\mu$ and $\nu$ compared with gradient-based optimization algorithms, and therefore has been applied in our HNAS framework. 
Specifically, in every iteration $k$ of Algorithm \ref{alg:hnas}, a GP belief with mean $u(\mu, \nu)$ and variance~$\sigma^2(\mu,\nu)$ for the entire input domain is firstly obtained following the Equation (1) in \citep{gp-ucb} (i.e., by letting input $x$ in \citep{gp-ucb} be the column vector $(\mu, \nu)^{\top}$ and the function value $y$ in \citep{gp-ucb} be $\Ls_{\text{val}}(\gA)$) using the historical evaluations $\gH_{k-1}=\{((\mu_i, \nu_i), \Ls_{\text{val}}(A^*_i))\}_{i=1}^{k-1}$ (this corresponds to line 6 in Algorithm \ref{alg:hnas} for iteration $k-1$). \footnote{Since BO is usually applied to solve maximization problem, we use the historical evaluations $\gH_{k-1}=\{((\mu_i, \nu_i), -\Ls_{\text{val}}(A^*_i))\}_{i=1}^{k-1}$ for BO instead in order to maximize $-\Ls_{\text{val}}(A)$ in practice.} 
Then, the mean $u(\mu, \nu)$ and standard deviation $\sigma(\mu,\nu)$ from the resulting GP belief are used to construct an acquisition function such as the expected improvement (EI) from \citep{ei} or the upper confidence bound (UCB) $u(\mu, \nu) + \sqrt{\beta} \sigma(\mu,\nu)$ from \citep{gp-ucb} where the parameter $\beta > 0$ is set to trade off between exploitation vs. exploration for guaranteeing no regret asymptotically with high probability.
Finally, an input (i.e., $\mu_k,\nu_k$) will be selected (for querying) by maximizing the acquisition function within the entire input domain (i.e., line 3 in Algorithm \ref{alg:hnas}), e.g., $(\mu_k,\nu_k)=\arg\max_{(\mu,\nu)} u(\mu, \nu) + \sqrt{\beta} \sigma(\mu,\nu)$ for UCB. The acquisition function in BO is usually differentiable and thus gradient-based optimization algorithms (e.g., L-BFGS and gradient ascent) can be applied to maximize it. We refer to \citep{gp-ucb} for more technical details about the BO algorithm based on UCB and \citep{bo-imp}
for the implementation of BO that has been used in our experiments.

\subsection{Experimental Details in NAS-Bench-201}\label{sec:app:exp-details}
In our experiments on NAS-Bench-201, we set the number of iterations $K$ for Algorithm \ref{alg:hnas} to be 20.
In addition, for every iteration of Algorithm \ref{alg:hnas}, we independently and randomly sample a pool of 2,000 architectures from the search space and then choose the architecture enjoying the optimum value of \eqref{eq:nas-final} from all sampled architectures (e.g., $2000 \times k$ architectures in total). After choosing this candidate architecture, we query the validation performance of this architecture on CIFAR-10 after 12-epoch training (i.e., ``hp=12'') from the tabular data in NAS-Bench-201, which then will be employed to update the GP surrogate applied in BO. After completing 20 iterations of our Algorithm \ref{alg:hnas}, there are \textit{(a)} 40,000 sampled architectures with evaluated training-free metrics which can already cover all the architectures in NAS-Bench-201 (consisting of 15,625 architectures) with a high probability, and  \textit{(b)} 20 architectures with evaluated validation performance which can already allow our HNAS to select architectures achieving competitive performances. Overall, our \eqref{eq:nas-final} and Algorithm \ref{alg:hnas} can be solved both efficiently and effectively following our aforementioned optimization techniques.

\section{More Empirical Results}\label{sec:app:empirical}
\subsection{Connections among Training-Free Metrics}\label{sec:app:connection}
\begin{table*}[t!]
\renewcommand\multirowsetup{\centering}
\caption{Connection between any two training-free metrics (i.e., $\gM_1$ and $\gM_2$ in the table) from Sec.~\ref{sec:metrics} in NAS-Bench-101/201. Note that each training-free metric is evaluated using a batch of randomly sampled data from CIFAR-10 following that of \citep{zero-cost}.}
\label{tab:app:connection}
\centering
\resizebox{\textwidth}{!}{
\begin{tabular}{llcccccc}
\toprule
\multirow{2}{*}{\textbf{$\gM_1$}} &
\multirow{2}{*}{\textbf{$\gM_2$}} &
\multicolumn{3}{c}{\textbf{NAS-Bench-101}} &
\multicolumn{3}{c}{\textbf{NAS-Bench-201}} \\
\cmidrule(l){3-5} \cmidrule(l){6-8} 
& & Pearson & Spearman & Kendall's Tau & Pearson & Spearman & Kendall's Tau \\
\midrule
& & \multicolumn{6}{c}{\textbf{Gradient-based training-free metrics}} \\
$\gM_{\gradnorm}$ & $\gM_{\snip}$ & 0.98 & 0.98 & 0.87 & 1.00 & 1.00 & 0.97 \\
$\gM_{\gradnorm}$ & $\gM_{\grasp}$ & 0.35 & 0.61 & 0.43 & 0.60 & 0.92 & 0.77 \\
$\gM_{\gradnorm}$ & $\gM_{\tracenorm}$ & 0.98 & 0.98 & 0.87 & 0.98 & 0.97 & 0.85 \\
$\gM_{\snip}$ & $\gM_{\grasp}$ & 0.34 & 0.59 & 0.42 & 0.55 & 0.92 & 0.77 \\
$\gM_{\snip}$ & $\gM_{\tracenorm}$ & 0.94 & 0.93 & 0.77 & 0.97 & 0.96 & 0.83 \\
$\gM_{\grasp}$ & $\gM_{\tracenorm}$ & 0.37 & 0.57 & 0.40 & 0.69 & 0.89 & 0.73 \\
\midrule
$\gM_{\knas}$ & $\gM_{\gradnorm}$ & 0.95 & 0.96 & 0.83 & 0.88 & 0.94 & 0.80\\
$\gM_{\knas}$ & $\gM_{\snip}$ & 0.91 & 0.92 & 0.75 & 0.87 & 0.94 & 0.78 \\
$\gM_{\knas}$ & $\gM_{\grasp}$ & 0.37 & 0.65 & 0.46 & 0.45 & 0.87 & 0.69 \\
$\gM_{\knas}$ & $\gM_{\tracenorm}$ & 0.96 & 0.96 & 0.84 & 0.89 & 0.97 & 0.86 \\
\midrule
& & \multicolumn{6}{c}{\textbf{Non-gradient-based training-free metrics}} \\
$\gM_{\fisher}$ & $\gM_{\tracenorm}$ & 0.69 & 0.97 & 0.85 & 0.30 & 0.78 & 0.69 \\
$\gM_{\synflow}$ & $\gM_{\tracenorm}$ & 0.02 & 0.50 & 0.34 & 0.07 & 0.49 & 0.35 \\
$\gM_{\naswot}$ & $\gM_{\tracenorm}$ & 0.08 & 0.11 & 0.08 & 0.10 & 0.32 & 0.22 \\
\bottomrule
\end{tabular}
}
\end{table*}

Besides the theoretical (Theorem~\ref{th:connection}) and empirical (Sec.~\ref{sec:connection}) connections between $\gM_{\tracenorm}$ and other gradient-based training-free metrics from Sec.~\ref{sec:metrics}, we further show in Table~\ref{tab:app:connection} that any two metrics from Sec.~\ref{sec:metrics} are highly correlated, i.e., they consistently achieve large positive correlations in both NAS-Bench-101 and NAS-Bench-201. Similar to the results in our Sec.~\ref{sec:connection}, the correlation between $\gM_{\grasp}$ and any other training-free metric is generally lower than other pairs, which may result from the hessian matrix that has only been applied in $\gM_{\grasp}$. To figure out \emph{whether our Theorem~\ref{th:connection} is also applicable to non-gradient-based training-free metrics}, we then provide the correlation between $\gM_{\fisher}$ \citep{fisher}, $\gM_{\synflow}$ \citep{synflow}, $\gM_{\naswot}$ \citep{naswot} and $\gM_{\tracenorm}$ \citep{nasi} for the comparison. 
Interestingly, both $\gM_{\fisher}$ and $\gM_{\synflow}$ achieve higher positive correlations with $\gM_{\tracenorm}$ than $\gM_{\naswot}$ in general. According to their mathematical forms in the corresponding papers, such a phenomenon may result from the fact that $\gM_{\fisher}$ and $\gM_{\synflow}$ have contained certain gradient information while $\gM_{\naswot}$ only relies on the outputs of each layer in an initialized architecture. \footnote{Of note, the so-called gradient information contained in $\gM_{\fisher}$ and $\gM_{\synflow}$ is different from the commonly used gradient of initialized model parameters that is derived from loss function or the output of DNN models. So, $\gM_{\fisher}$ and $\gM_{\synflow}$ are taken as the non-gradient-based training-free metrics instead in this paper.} These results therefore imply that our Theorem~\ref{th:connection} may also provide valid theoretical connections for the training-free metrics that are not gradient-based but still contain certain gradient information.

\subsection{Valid Generalization Guarantees for Training-Free NAS}\label{sec:app:generalization}
\begin{figure}[t]
\centering
\begin{tabular}{c}
    \hspace{-2mm}\includegraphics[width=\columnwidth]{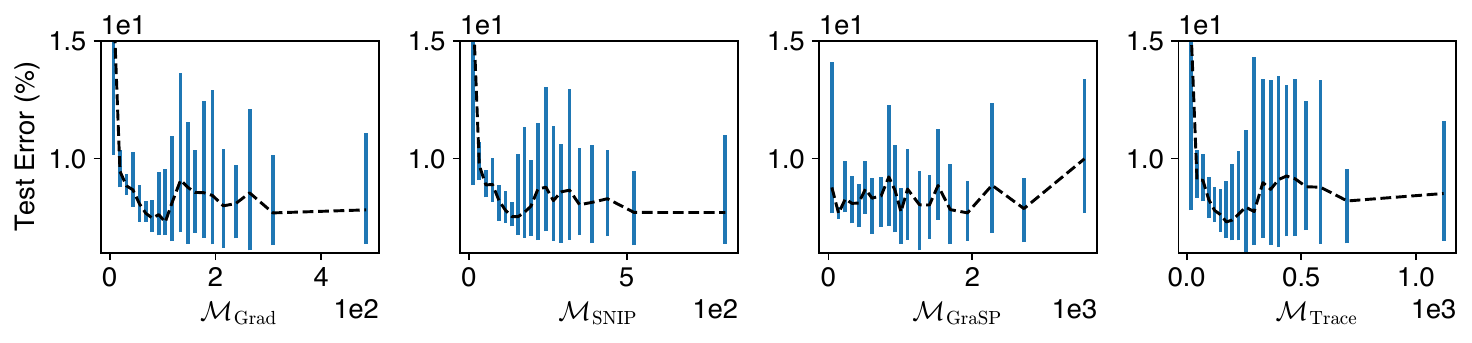} \\
    {(a) Varying architecture performances in the search space} \\
    \hspace{-2mm}\includegraphics[width=\columnwidth]{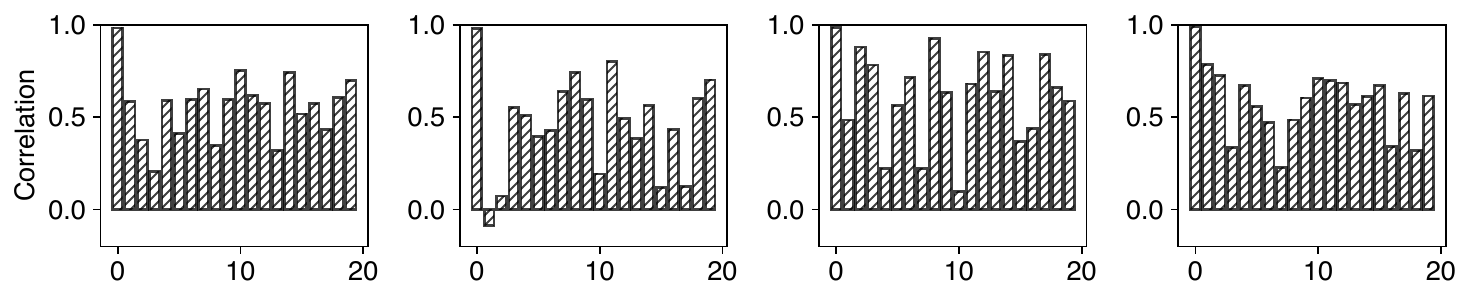} \\
    {(b) Correlation between condition numbers and architecture performances}
\end{tabular}
\caption{(a) Varying architecture performances under different value of training-free metrics in NAS-Bench-201. Note that the $x$-axis denotes the averaged value of training-free metrics over the architectures grouped in the same bin and $y$-axis denoted the test error evaluated on CIFAR-10. 
(b) Correlation between the condition numbers and the true generalization performances of the architectures within the same bin (i.e., the $y$-axis). Note that the $x$-axis denotes the corresponding 20 bins in Figure~\ref{fig:app:corollary}~(a).}
\label{fig:app:corollary}
\end{figure}

To further support that our Corollary~\ref{corol:generalization-non-realizable} presents a more practical and valid generalization guarantee for training-free NAS in practice, we examine the true generalization performances of all candidate architectures under their different value of training-free metrics in Figure~\ref{fig:app:corollary}~(a) and exhibit the correlation between the condition number and the true generalization performances of all candidate architectures in Figure~\ref{fig:app:corollary}~(b). Specifically, we group the value of training-free metrics in NAS-Bench-201 into 20 bins and then plot the test errors on CIFAR-10 of all candidate architectures within the same bin into the blue vertical lines in Figure~\ref{fig:app:corollary}~(a). Besides, we plot the averaged test errors over the architectures within the same bin into the black dash lines in Figure~\ref{fig:app:corollary}~(a). Besides, each correlation between condition number and test error in Figure~\ref{fig:app:corollary}~(b) is computed using the candidate architectures within the same bin. 

Notably, as illustrated by the black dash lines in Figure~\ref{fig:app:corollary}~(a), there consistently exists a trade-off for all the training-free metrics in Sec.~\ref{sec:metrics}. Specifically, there exists an optimal value $\gM_{\text{opt}}$ for each training-free metric $\gM$ that is capable of achieving the best generalization performance in the search space. When $\gM < \gM_{\text{opt}}$, architecture with a larger value of $\gM$ typically enjoys a better generalization performance. On the contrary, when $\gM > \gM_{\text{opt}}$, architecture with a smaller value of $\gM$ generally achieves a better generalization performance. Interestingly, these results perfectly align with our Corollary~\ref{corol:generalization-non-realizable}. Furthermore, Figure~\ref{fig:app:corollary}~(b) shows that the condition number is indeed highly correlated to the generalization performance of candidate architectures and a smaller condition number is generally preferred in order to select well-performing architectures in training-free NAS. More interestingly, similar phenomenons can also be found in \citep{nasi} and \citep{te-nas}. Remarkably, our Corollary~\ref{corol:generalization-non-realizable} can provide theoretically grounded interpretations for these results, whereas Corollary~\ref{corol:generalization-realizable} fails to characterize these phenomenons. Consequently, our Corollary~\ref{corol:generalization-non-realizable} is shown to be more practical and valid in practice.

Based on the conclusions above, we then compare the impacts of the trade-off and condition number $\kappa$ mentioned above by examining the correlation between the true generalization performances of candidate architectures and their training-free metrics applied in different scenarios. Here, we use the same parameters applied in Sec. \ref{sec:exp:generalization} for Corollary~\ref{corol:generalization-non-realizable}. Table~\ref{tab:app:generalization} summarizes the comparison. Note that the non-realizable scenario is equivalent to the realizable scenario + trade-off + $\kappa$ as suggested by our Corollary~\ref{corol:generalization-non-realizable}. As revealed in Table~\ref{tab:app:generalization}, both trade-off and condition number $\kappa$ are necessary to achieve an improved characterization of architecture performances over the one in the realizable scenario followed by \citep{zero-cost}, which again verifies the practicality and validity of our Corollary~\ref{corol:generalization-non-realizable}. More interestingly, condition number $\kappa$ is shown to be more essential than the trade-off for training-free NAS in order to improve the correlations in the realizable scenario. By integrating both trade-off and condition number $\kappa$ into the realizable scenario, the non-realizable scenario consistently enjoys the highest correlations on different datasets, which also further verifies the improvement of our training-free NAS objective \eqref{eq:nas-final} over the one used in \citep{zero-cost}.

\begin{table}[t]
\renewcommand\multirowsetup{\centering}
\caption{Correlation between the test errors of candidate architectures in NAS-Bench-201 and their training-free metrics applied in several different scenarios. We refer to Sec.~\ref{sec:generalization-nas} for more details about the trade-off and condition number $\kappa$ applied in the following scenarios.}
\label{tab:app:generalization}
\centering
\resizebox{\columnwidth}{!}{
\begin{tabular}{llcccccccc}
\toprule
\multirow{2}{*}{\textbf{Dataset}} &
\multirow{2}{*}{\textbf{Scenario}} &
\multicolumn{4}{c}{\textbf{Spearman}} &
\multicolumn{4}{c}{\textbf{Kendall's Tau}} \\
\cmidrule(l){3-6} \cmidrule(l){7-10}
& & $\gM_{\gradnorm}$ &
$\gM_{\snip}$ &
$\gM_{\grasp}$ &
$\gM_{\tracenorm}$  & 
$\gM_{\gradnorm}$ &
$\gM_{\snip}$ &
$\gM_{\grasp}$ &
$\gM_{\tracenorm}$ \\
\midrule
\multirow{4}{*}{C10} & Realizable & 0.637 & 0.639 & 0.566 & 0.538 & 0.469 & 0.472 & 0.400 & 0.387 \\
& Realizable + Trade-off & 0.642 & 0.641 & 0.570 & 0.549 & 0.475 & 0.474 & 0.403 & 0.397 \\
& Realizable + $\kappa$ & 0.724 & 0.728 & 0.658 & 0.657 & 0.530 & 0.533 & 0.474 & 0.474 \\
& Non-realizable & \textbf{0.750} & \textbf{0.748} & \textbf{0.686} & \textbf{0.697} & \textbf{0.559} & \textbf{0.556} & \textbf{0.501} & \textbf{0.512} \\
\midrule
\multirow{4}{*}{C100} & Realizable & 0.638 & 0.638 & 0.571 & 0.535 & 0.473 & 0.475 & 0.409 & 0.385 \\
& Realizable + Trade-off & 0.642 & 0.645 & 0.578 & 0.546 & 0.476 & 0.481 & 0.414 & 0.394 \\
& Realizable + $\kappa$ & 0.716 & 0.719 & 0.649 & 0.651 & 0.527 & 0.529 & 0.469 & 0.470 \\
& Non-realizable & \textbf{0.740} & \textbf{0.746} & \textbf{0.680} & \textbf{0.686} & \textbf{0.552} & \textbf{0.557} & \textbf{0.498} & \textbf{0.504} \\
\midrule
\multirow{4}{*}{IN-16} & Realizable & 0.578 & 0.578 & 0.550 & 0.486 & 0.430 & 0.433 & 0.397 & 0.354 \\
& Realizable + Trade-off & 0.588 & 0.589 & 0.566 & 0.526 & 0.438 & 0.441 & 0.408 & 0.382 \\
& Realizable + $\kappa$ & 0.646 & 0.649 & 0.612 & 0.587 & 0.472 & 0.474 & 0.443 & 0.423 \\
& Non-realizable & \textbf{0.682} & \textbf{0.685} & \textbf{0.655} & \textbf{0.660} & \textbf{0.505} & \textbf{0.506} & \textbf{0.480} & \textbf{0.482} \\
\bottomrule
\end{tabular}
}
\end{table}

\subsection{Transferability of Training-Free NAS}
\begin{table}[t]
\renewcommand\multirowsetup{\centering}
\caption{
Deviation of the correlation between the test errors in NAS-Bench-201 and the generalization bounds in Sec. \ref{sec:generalization-nas} using training-free metrics evaluated on various datasets.
Each correlation is reported with the mean and standard deviation using the metrics evaluated on CIFAR-10/100 and ImageNet-16-120. 
Small standard deviations imply strong transferability.
}
\label{tab:transferability}
\centering
\begin{tabular}{L{1.2cm}lcccc}
\toprule
\multirow{2}{*}{\textbf{Dataset}} &
\multicolumn{4}{c}{\textbf{Training-free Metrics}} \\
\cmidrule(l){2-5}
& $\gM_{\gradnorm}$ &
$\gM_{\snip}$ &
$\gM_{\grasp}$ &
$\gM_{\tracenorm}$ \\
\midrule
\multicolumn{5}{c}{\textbf{Realizable scenario}} \\
C10  & 0.64$\pm$0.01 & 0.64$\pm$0.01 & 0.58$\pm$0.02 & 0.55$\pm$0.01 \\
C100 & 0.64$\pm$0.01 & 0.64$\pm$0.01 & 0.58$\pm$0.03 & 0.54$\pm$0.02 \\
IN-16 & 0.57$\pm$0.01 & 0.57$\pm$0.01 & 0.52$\pm$0.03 & 0.47$\pm$0.02 \\
\midrule
\multicolumn{5}{c}{\textbf{Non-realizable scenario}} \\
C10 & 0.75$\pm$0.00 & 0.75$\pm$0.00 & 0.69$\pm$0.01 & 0.69$\pm$0.00 \\
C100 & 0.74$\pm$0.00 & 0.74$\pm$0.00 & 0.69$\pm$0.01 & 0.69$\pm$0.01 \\
IN-16 & 0.69$\pm$0.00 & 0.69$\pm$0.00 & 0.63$\pm$0.01 & 0.65$\pm$0.00 \\
\bottomrule
\end{tabular}
\vskip -0.1in
\end{table}

In practice, the transferability of the architectures selected by both training-based and training-free NAS algorithms has been widely verified \citep{darts, te-nas, nasi}.
So, in this section, we also verify the transferability of our generalization guarantees for training-free NAS. Specifically, we examine the deviation of the correlation between the architecture performance and the generalization bounds in Sec.~\ref{sec:generalization-nas} using training-free metrics evaluated on different datasets. That is, training-free metrics and architecture performance usually will be evaluated on different datasets.
Table~\ref{tab:transferability} summarizes the results using CIFAR-10/100 (C10/100) and ImageNet-16-120 (IN-16) \citep{imagenet-16-120} in NAS-Bench-201 where we employ the same parameters as Sec. \ref{sec:exp:generalization} for Corollary \ref{corol:generalization-non-realizable}. Notably, nearly the same correlations (i.e., with extremely small deviations) are achieved for training-free metrics evaluated on different datasets. This implies that the training-free metrics computed on a dataset $S$ can also provide a good characterization of the architecture performance evaluated on another dataset $S'$. Therefore, the architectures selected by training-free NAS algorithms on $S$ are also likely to produce a compelling performance on $S'$. That is, the transferability of the architectures selected by training-free NAS is guaranteed.

\subsection{Additional Comparison in NAS-Bench-201}\label{sec:app:nasbench201}
In addition to the comparison of search performances and search costs (measured by GPU seconds) in Table \ref{tab:sota-nasbench201}, we further provide the comparison of the number of queries required by different NAS algorithms in Table \ref{tab:sota-nasbench201-queries}. The queries compared here are applied to evaluate the validation performance of the selected architectures after training, which is typically avoided by training-free NAS algorithms. Consequently, here, we mainly compare HNAS with other training-based NAS algorithms. As shown in Table \ref{tab:sota-nasbench201-queries}, HNAS can consistently achieve improved search performances with fewer number of queries, which also aligns with the results in our Table \ref{tab:sota-nasbench201}. This therefore further confirms the superior search efficiency and the remarkable search effectiveness of our HNAS framework.

\begin{table}[t]
\caption{Comparison of the number of queries (to evaluate the validation performances of trained architectures) required by different NAS algorithms in NAS-Bench-201. The performance of each algorithm is reported with the mean and standard deviation of five independent searches.}
\label{tab:sota-nasbench201-queries}
\centering
\begin{tabular}{lccccccc}
\toprule
\multirow{2}{*}{\textbf{Algorithm}} & \multicolumn{3}{c}{\textbf{Test Accuracy (\%)}} &
\multirow{2}{*}{\textbf{$\#$ Queries}}\\
\cmidrule(l){2-4} 
& C10 & C100 & IN-16 & \\
\midrule
REA & 93.92$\pm$0.30 & 71.84$\pm$0.99 & 45.15$\pm$0.89 & 102 \\
RS (w/o sharing) & 93.70$\pm$0.36 & 71.04$\pm$1.07 & 44.57$\pm$1.25 & 106 \\
REINFORCE & 93.85$\pm$0.37 & 71.71$\pm$1.09 & 45.24$\pm$1.18 & 103  \\
\midrule
HNAS ($\gM_{\gradnorm}$) & \textbf{94.04}$\pm$0.21 & 71.75$\pm$1.04 & \textbf{45.91}$\pm$0.88 & \textbf{20} \\
HNAS ($\gM_{\snip}$) & \textbf{93.94}$\pm$0.02 & 71.49$\pm$0.11 & \textbf{46.07}$\pm$0.14 & \textbf{20} \\
HNAS ($\gM_{\grasp}$) & \textbf{94.13}$\pm$0.13 & \textbf{72.59}$\pm$0.82 & \textbf{46.24}$\pm$0.38 & \textbf{20} \\
HNAS ($\gM_{\tracenorm}$) & \textbf{94.07}$\pm$0.10 & \textbf{72.30}$\pm$0.70 & \textbf{45.93}$\pm$0.37 & \textbf{20} \\
\midrule
\textbf{Optimal} & 94.37 & 73.51 & 47.31 & - \\
\bottomrule
\end{tabular}
\end{table}

\subsection{HNAS in the DARTS Search Space}\label{sec:app:sota-darts}
To support the effectiveness and efficiency of our HNAS, we also apply HNAS in the DARTS \citep{darts} search space to find well-performing architectures on CIFAR-10/100 and ImageNet \citep{imagenet}. Specifically, we sample a pool of 60000 architecture to evaluate their training-free metrics on CIFAR-10 in order to maintain high computational efficiency for these training-free metrics. For the results on CIFAR-10/100, we then apply the BO algorithm for 25 iterations with a 10-epoch model training for the selected architectures in our HNAS (Algorithm~\ref{alg:hnas}). As for the results on ImageNet, we apply the BO algorithm for 10 iterations with a 3-epoch model training for the selected architectures in our HNAS. We follow \citep{darts} to construct 20-layer final selected architectures with an auxiliary tower of weight 0.4 for CIFAR-10 (0.6 for CIFAR-100) located at 13-th layer and 36 initial channels. We evaluate these architectures on CIFAR-10/100 using stochastic gradient descent (SGD) of 600 epochs with a learning rate cosine scheduled from 0.025 to 0 for CIFAR-10 (from 0.035 to 0.001 for CIFAR-100), momentum 0.9, weight decay 3$\times$10$^{-4}$and batch size 96. Both Cutout \citep{cutout}, and ScheduledDropPath linearly increased from 0 to 0.2 for CIFAR-10 (from 0 to 0.3 for CIFAR-100) are employed for regularization purposes on CIFAR-10/100. As for the evaluation on ImageNet,
we train the 14-layer architecture from scratch for 250 epochs with a batch size of 1024. The learning rate is warmed up to 0.7 for the first 5 epochs and then decreased to zero with a cosine schedule. We adopt the SGD optimizer with 0.9 momentum and a weight decay of 3$\times$10$^{-5}$.

The results on CIFAR-10/100 and ImageNet are summarized in Table~\ref{tab:app:cifar} and Table~\ref{tab:app:imagenet}, respectively. As shown in Table~\ref{tab:app:cifar}, both our HNAS (C10) and HNAS (C100) are capable of achieving state-of-the-art performance on CIFAR-10 and CIFAR-100, correspondingly, while incurring lower search costs than other training-based NAS algorithms. Even compared with other training-free NAS baselines, e.g., TE-NAS, our HNAS can still enjoy a compelling search cost. Overall, these results further validate that our HNAS is indeed able to enjoy the superior search efficiency of training-free NAS and also the remarkable search effectiveness of training-based NAS. More interestingly, our HNAS (C10) can achieve a lower test error on CIFAR-10 but a higher test error on CIFAR-100 when compared with HNAS (C100). This result indicates that similar to training-based NAS algorithms, directly searching on the target dataset is also able to improve the final performance in HNAS. By exploiting this advantage over other training-free NAS baselines, our HNAS thus is capable of selecting architectures achieving higher performances, as shown in Table~\ref{tab:app:cifar}. Similar results are also achieved on ImageNet as shown in Table~\ref{tab:app:imagenet}. Overall, these results have further supported the superior search efficiency and remarkable search effectiveness of our HNAS that we have verified in Sec.~\ref{sec:exp:hnas}.

\begin{table*}[t]
\caption{Performance comparison among state-of-the-art (SOTA) neural architectures on CIFAR-10/100. The performance of the final architectures selected by HNAS is reported with the mean and standard deviation of five independent evaluations. The search costs are evaluated on a single Nvidia 1080Ti. Note that HNAS (C10 or C100) denoted the architecture selected by our HNAS using the dataset CIFAR-10 or CIFAR-100, respectively.}
\label{tab:app:cifar}
\centering
\resizebox{\textwidth}{!}{
\begin{threeparttable}
\begin{tabular}{lcccccc}
\toprule
\multirow{2}{*}{\textbf{Algorithm}} & \multicolumn{2}{c}{\textbf{Test Error (\%)}} &
\multicolumn{2}{c}{\textbf{Params} (M)} &
\multirow{2}{2cm}{\textbf{Search Cost} (GPU Hours)} &
\multirow{2}{*}{\textbf{Search Method}} \\
\cmidrule(l){2-3} \cmidrule(l){4-5} 
& C10 & C100 & C10 & C100 & \\
\midrule 
DenseNet-BC \citep{densenet} & 3.46$^*$ & 17.18$^*$ & 25.6 & 25.6 & - & manual\\
\midrule 
NASNet-A \citep{nasnet} & 2.65 & - & 3.3 & - & 48000 & RL\\
AmoebaNet-A \citep{amoebanet} & 3.34$\pm$0.06 & 18.93$^\dagger$ & 3.2 & 3.1 & 75600 & evolution\\
PNAS \citep{pnas} & 3.41$\pm$0.09 & 19.53$^*$ & 3.2 & 3.2 & 5400 & SMBO\\
ENAS \citep{enas} & 2.89 & 19.43$^*$ & 4.6 & 4.6 & 12 & RL\\
NAONet \citep{naonet} & 3.53 & - & 3.1 & - & 9.6 & NAO\\
\midrule
DARTS (2nd) \citep{darts} & 2.76$\pm$0.09 & 17.54$^\dagger$ & 3.3 & 3.4 & 24 & gradient\\
GDAS \citep{gdas} & 2.93 & 18.38 & 3.4 & 3.4 & 7.2 & gradient\\
NASP \citep{nasp} & 2.83$\pm$0.09 & - & 3.3 & - & 2.4 & gradient\\
P-DARTS \citep{p-darts} & 2.50 & - & 3.4 & - & 7.2 & gradient\\
DARTS- (avg) \citep{darts-} & 2.59$\pm$0.08 & 17.51$\pm$0.25 & 3.5 & 3.3 & 9.6 & gradient\\
SDARTS-ADV \citep{sdarts} & 2.61$\pm$0.02 & - & 3.3 & - & 31.2 & gradient\\
R-DARTS (L2) \citep{r-darts}  & 2.95$\pm$0.21 & 18.01$\pm$0.26 & - & - & 38.4 & gradient\\
DrNAS \citep{drnas} & 2.46$\pm$0.03 & - & 4.1 & - & 14.4 & gradient \\
\midrule
TE-NAS$^{\sharp}$ \citep{te-nas} & 2.83$\pm$0.06 & 17.42$\pm$0.56 & 3.8 & 3.9 & 1.2 & training-free \\
NASI-ADA \cite{nasi} & 2.90$\pm$0.13 & 16.84$\pm$0.40 & 3.7 & 3.8 & 0.24 & training-free \\
\midrule
HNAS (C10) & 2.62$\pm$0.04 & 17.10$\pm$0.18 & 3.4 & 3.5 & 2.4 & hybrid \\
HNAS (C100) & 2.78$\pm$0.05 & \textbf{16.29}$\pm$0.14 & 3.7 & 3.8 & 2.7 & hybrid \\
\bottomrule
\end{tabular}
\begin{tablenotes}\footnotesize
    \item[$\dagger$] Reported by \citet{gdas} with their experimental settings.
    \item[$*$] Obtained by training corresponding architectures without cutout \citep{cutout} augmentation.
    \item[$\sharp$] Reported by \citet{nasi} with their experimental settings.
\end{tablenotes}
\end{threeparttable}
}
\end{table*}

\begin{table}[t]
\renewcommand\multirowsetup{\centering}
\caption{Performance comparison among SOTA image classifiers on ImageNet.}
\label{tab:app:imagenet}
\centering
\begin{tabular}{lccccc}
\toprule
\multirow{2}{*}{\textbf{Algorithm}} & \multicolumn{2}{c}{\textbf{Test Error (\%)}} &
\multirow{2}{1.0cm}{\textbf{Params} (M)} &
\multirow{2}{0.6cm}{\textbf{$+\times$} (M)} & 
\multirow{2}{2cm}{\textbf{Search Cost}  (GPU Days)} \\
\cmidrule(l){2-3}
& Top-1 & Top-5 & & \\
\midrule 
Inception-v1 \citep{inception}  & 30.1 & 10.1 & 6.6 & 1448 & - \\
MobileNet \citep{mobilenet} & 29.4 & 10.5 & 4.2 & 569 & - \\
ShuffleNet 2$\times$(v2) \citep{shufflenetv2} & 25.1 & 7.6 & 7.4 & 591 & - \\
\midrule
NASNet-A \citep{nasnet} & 26.0 & 8.4 & 5.3 & 564 & 2000\\
AmoebaNet-A \citep{amoebanet} & 25.5 & 8.0 & 5.1 & 555 & 3150 \\
PNAS \citep{pnas} & 25.8 & 8.1 & 5.1 & 588 & 225 \\
MnasNet-92 \citep{mnasnet} & 25.2 & 8.0 & 4.4 & 388 & - \\
\midrule
DARTS \citep{darts} & 26.7 & 8.7 & 4.7 & 574 & 4.0 \\
SNAS (mild) \citep{snas} & 27.3 & 9.2 & 4.3 & 522 & 1.5\\
GDAS \citep{gdas} & 26.0 & 8.5 & 5.3 & 581 & 0.21\\
ProxylessNAS \citep{proxyless-nas} & 24.9 & 7.5 & 7.1 & 465 & 8.3 \\
DARTS- \citep{darts-} & 23.8 & 7.0 & 4.5 & 467 & 4.5 \\
SDARTS-ADV \citep{sdarts} & 25.2 & 7.8 & 5.4 & 594 & 1.3 \\
DrNAS \citep{drnas} & 23.7 & 7.1 & 5.7 & - & 4.6 \\
\midrule
TE-NAS (C10) \citep{te-nas} & 26.2 & 8.3 & 5.0 & - & 0.05 \\
TE-NAS (ImageNet) \citep{te-nas} & 24.5 & 7.5 & 5.4 & - & 0.17 \\
NASI-ADA \citep{nasi} & 25.0 & 7.8 & 4.9 & 559 & 0.01 \\
\midrule
HNAS (C100) & 24.8 & 7.8 & 5.2 & 601 & 0.1 \\
HNAS (ImageNet) & 24.3 & 7.4 & 5.1 & 575 & 0.5 \\

\bottomrule
\end{tabular}
\end{table}
\subsection{Ablation Studies}\label{sec:app:ablation}

\paragraph{Ablation Study on Initialization Method.} 
While our theoretical analyses throughout this paper are based on the initialization using the standard normal distribution (Sec.~\ref{sec:notations}), \footnote{Note that this initialization is equivalent to the LeCun initialization \citep{init-lecun} according to \citep{ntk}.} we wonder \emph{whether our theoretical results are also applicable to DNNs using different initialization methods}, e.g., Xavier \citep{init-xavier} and Kaiming \citep{init-kaiming} initialization. Specifically, we compare the correlation between the true generalization performances of all candidate architectures in NAS-Bench-201 and the generalization guarantees in Sec.~\ref{sec:generalization-nas} that are evaluated using different initialization methods. Table~\ref{tab:app:generalization-init} summarizes the comparison. Here, we use the same parameters applied in Sec. \ref{sec:exp:generalization} for Corollary~\ref{corol:generalization-non-realizable}. Notably, Table~\ref{tab:app:generalization-init} shows that our generalization guarantees for training-free NAS, i.e., Corollary~\ref{corol:generalization-realizable},~\ref{corol:generalization-non-realizable}, can also perform well for training-free NAS using DNNs initialized with different methods, indicating a wider application of our generalization guarantees in Sec.~\ref{sec:generalization-nas}. Of note, LeCun initialization can achieve the best results among the three initialization methods in Table~\ref{tab:app:generalization-init} since it satisfies our assumption about the initialization of DNNs. As an implication, LeCun initialization is more preferred when using the training-free metrics from Sec.~\ref{sec:metrics} to characterize the architecture performances in training-free NAS.

\begin{table}[t]
\renewcommand\multirowsetup{\centering}
\caption{Correlation between the test errors (on CIFAR-10) of all architectures in NAS-Bench-201 and our generalization guarantees in Sec.~\ref{sec:generalization-nas} that are evaluated on DNNs using different initialization methods.}
\label{tab:app:generalization-init}
\centering
\resizebox{\columnwidth}{!}{
\begin{tabular}{lcccccccc}
\toprule
\multirow{2}{*}{\textbf{Initialization}} &
\multicolumn{4}{c}{\textbf{Spearman}} &
\multicolumn{4}{c}{\textbf{Kendall's Tau}} \\
\cmidrule(l){2-5} \cmidrule(l){6-9}
& $\gM_{\gradnorm}$ &
$\gM_{\snip}$ &
$\gM_{\grasp}$ &
$\gM_{\tracenorm}$  & 
$\gM_{\gradnorm}$ &
$\gM_{\snip}$ &
$\gM_{\grasp}$ &
$\gM_{\tracenorm}$ \\
\midrule
& \multicolumn{8}{c}{\textbf{Realizable scenario}}\\
LeCun \citep{init-lecun} & 0.637 & 0.639 & 0.566 & 0.538 & 0.469 & 0.472 & 0.400 & 0.387 \\
Xavier \citep{init-xavier} & 0.608 & 0.627 & 0.449 & 0.465 & 0.445 & 0.463 & 0.316 & 0.334 \\
He \citep{init-kaiming} & 0.609 & 0.615 & 0.340 & 0.460 & 0.446 & 0.454 & 0.242 & 0.334 \\
\midrule
& \multicolumn{8}{c}{\textbf{Non-realizable scenario}}\\
LeCun \citep{init-lecun} & 0.750 & 0.748 & 0.686 & 0.697 & 0.559 & 0.556 & 0.501 & 0.512 \\
Xavier \citep{init-xavier} & 0.676 & 0.685 & 0.615 & 0.635 & 0.493 & 0.501 & 0.442 & 0.460 \\
He \citep{init-kaiming} & 0.607 & 0.611 & 0.505 & 0.569 & 0.436 & 0.439 & 0.358 & 0.407 \\
\bottomrule
\end{tabular}
}
\end{table}

\paragraph{Ablation Study on Batch Size.} 
Theoretically, the training-free metrics from Sec.~\ref{sec:metrics} are defined over the whole training dataset. In practice, we usually only apply a batch of randomly sampled data points to evaluate these training-free metrics in order to achieve a desirable computational efficiency, which follows \citep{zero-cost}. To investigate the impact of batch size on these metrics, we examine the correlation between the true generalization performances of all candidate architectures in NAS-Bench-201 and their generalization guarantees in the non-realizable scenario under varying batch sizes. Table~\ref{tab:app:generalization-batchsize} summarizes the results. Here, we use the same parameters applied in Sec. \ref{sec:exp:generalization} for Corollary~\ref{corol:generalization-non-realizable}. Besides the impact of batch size on training-free metrics, we also include the impact of batch size on condition number $\kappa$ in this table. Specifically, in the upper part of Table~\ref{tab:app:generalization-batchsize}, the correlations are evaluated using a batch size of 64 for $\kappa$ and varying batch sizes for any training-free metric $\gM$ from Sec. \ref{sec:metrics}. Meanwhile, in the lower part of Table~\ref{tab:app:generalization-batchsize}, the correlations are evaluated using varying batch sizes for $\kappa$ and a batch size of 64 for any training-free metric $\gM$. Notably, Table~\ref{tab:app:generalization-batchsize} shows that similar results will be achieved even when training-free metrics are evaluated under varying batch sizes, whereas $\kappa$ evaluated under varying batch sizes will lead to different results, indicating that $\kappa$ is more sensitive to batch size than training-free metrics. As an implication, while a small batch size is also able to perform well in practice, a large batch size is more preferred when using our generalization guarantees for training-free NAS.

\begin{table}[t]
\renewcommand\multirowsetup{\centering}
\caption{Correlation between the test errors (on CIFAR-10) of all architectures in NAS-Bench-201 and their generalization guarantees in the non-realizable scenario under varying batch size.}
\label{tab:app:generalization-batchsize}
\centering
\begin{tabular}{ccccccccc}
\toprule
\multirow{2}{*}{\textbf{Batch Size}} &
\multicolumn{4}{c}{\textbf{Spearman}} &
\multicolumn{4}{c}{\textbf{Kendall's Tau}} \\
\cmidrule(l){2-5} \cmidrule(l){6-9}
& $\gM_{\gradnorm}$ &
$\gM_{\snip}$ &
$\gM_{\grasp}$ &
$\gM_{\tracenorm}$  & 
$\gM_{\gradnorm}$ &
$\gM_{\snip}$ &
$\gM_{\grasp}$ &
$\gM_{\tracenorm}$ \\
\midrule
& \multicolumn{8}{c}{\textbf{Batch size 64 for $\kappa$ and varying batch sizes for any $\gM$}} \\
$\ms$4 & 0.737 & 0.741 & 0.671 & 0.684 & 0.547 & 0.550 & 0.487 & 0.501 \\
$\ms$8 & 0.739 & 0.743 & 0.676 & 0.689 & 0.549 & 0.552 & 0.492 & 0.506 \\
16 & 0.747 & 0.748 & 0.685 & 0.690 & 0.556 & 0.556 & 0.499 & 0.507 \\
32 & 0.750 & 0.748 & 0.687 & 0.690 & 0.558 & 0.556 & 0.502 & 0.506 \\
64 & 0.750 & 0.748 & 0.686 & 0.697 & 0.559 & 0.556 & 0.501 & 0.512 \\
\midrule
& \multicolumn{8}{c}{\textbf{Varying batch sizes for $\kappa$ and batch size 64 for any $\gM$}} \\
$\ms$4 & 0.578 & 0.585 & 0.569 & 0.509 & 0.416 & 0.421 & 0.402 & 0.362 \\
$\ms$8 & 0.597 & 0.603 & 0.591 & 0.542 & 0.429 & 0.433 & 0.419 & 0.386 \\
16 & 0.628 & 0.633 & 0.620 & 0.582 & 0.462 & 0.455 & 0.442 & 0.414 \\
32 & 0.663 & 0.666 & 0.645 & 0.621 & 0.479 & 0.481 & 0.462 & 0.445 \\
64 & 0.750 & 0.748 & 0.686 & 0.697 & 0.559 & 0.556 & 0.501 & 0.512 \\
\bottomrule
\end{tabular}
\end{table}

\paragraph{Ablation Study on Layer Width.} 
While our theoretical analyses are based on over-parameterized DNNs, i.e., $n>N$ in our Theorem~\ref{th:generalization}, we are also curious about \emph{how the layer width will influence our empirical results}. In particular, we examine the correlation between the true generalization performances of all candidate architectures in NAS-Bench-201 and their generalization guarantee in the non-realizable scenario under varying layer width. Similar to the ablation study on batch size, we investigate the impacts of layer width on the training-free metrics from Sec. \ref{sec:metrics} and the condition number $\kappa$ separately. Table~\ref{tab:app:generalization-channels} summarizes the results. Here, we use the same parameters applied in Sec. \ref{sec:exp:generalization} for Corollary~\ref{corol:generalization-non-realizable}. As shown in Table~\ref{tab:app:generalization-channels}, our generalization guarantee in the non-realizable scenario also performs well when layer width becomes smaller. Surprisingly, similar results can be achieved for training-free metrics evaluated under varying layer widths, whereas a larger layer width for training-free metrics typically leads to marginally higher correlations in Table~\ref{tab:app:generalization-channels}. On the contrary, a larger layer width for $\kappa$ leads to lower correlations in Table~\ref{tab:app:generalization-channels}. This may result from the similar behavior that can be achieved by layer width and topology width since both layer width and topology width are used to measure the width of DNN but in totally different perspectives. Therefore, increasing layer width will make deep architectures (in terms of topology) more indistinguishable from wide architectures (in terms of topology) and hence make it harder to apply our generalization guarantee in Corollary~\ref{corol:generalization-non-realizable} to characterize the architecture performances in a search space. As an implication, a large layer width for training-free metrics and a smaller layer width for condition number $\kappa$ are more preferred when applying our generalization guarantees for training-free NAS in practice.

\begin{table}[t]
\renewcommand\multirowsetup{\centering}
\caption{Correlation between the test errors (on CIFAR-10) of all architectures in NAS-Bench-201 and their generalization guarantees in the non-realizable scenario under varying layer widths, which are measured by the number of initial channels in our experiments. Larger initial channels indicates a large layer width.}
\label{tab:app:generalization-channels}
\centering
\begin{tabular}{ccccccccc}
\toprule
\multirow{2}{*}{\textbf{Init Channels}} &
\multicolumn{4}{c}{\textbf{Spearman}} &
\multicolumn{4}{c}{\textbf{Kendall's Tau}} \\
\cmidrule(l){2-5} \cmidrule(l){6-9}
& $\gM_{\gradnorm}$ &
$\gM_{\snip}$ &
$\gM_{\grasp}$ &
$\gM_{\tracenorm}$  & 
$\gM_{\gradnorm}$ &
$\gM_{\snip}$ &
$\gM_{\grasp}$ &
$\gM_{\tracenorm}$ \\
\midrule
& \multicolumn{8}{c}{\textbf{4 channels for $\kappa$ and varying channels for any $\gM$}} \\
$\ms$4 & 0.744 & 0.746 & 0.688 & 0.732 & 0.550 & 0.552 & 0.499 & 0.539 \\
$\ms$8 & 0.750 & 0.753 & 0.707 & 0.744 & 0.556 & 0.559 & 0.515 & 0.550 \\
16 & 0.753 & 0.753 & 0.728 & 0.750 & 0.558 & 0.559 & 0.535 & 0.556 \\
32 & 0.755 & 0.756 & 0.736 & 0.752 & 0.560 & 0.562 & 0.543 & 0.558 \\
\midrule
& \multicolumn{8}{c}{\textbf{Varying channels for $\kappa$ and 32 channels for any $\gM$}} \\
$\ms$4 & 0.755 & 0.756 & 0.736 & 0.752 & 0.560 & 0.562 & 0.543 & 0.558 \\
$\ms$8 & 0.720 & 0.722 & 0.700 & 0.709 & 0.529 & 0.531 & 0.512 & 0.522 \\
16 & 0.698 & 0.700 & 0.677 & 0.681 & 0.511 & 0.514 & 0.492 & 0.498 \\
32 & 0.686 & 0.688 & 0.664 & 0.664 & 0.501 & 0.503 & 0.481 & 0.484 \\
\bottomrule
\end{tabular}
\end{table}

\paragraph{Ablation Study on Generalization Guarantees and HNAS Using Non-Gradient-Based Training-Free Metrics.} 
As Appendix~\ref{sec:app:connection} has validated that our Theorem~\ref{th:connection} may also provide valid theoretical connections for certain non-gradient-based training-free metrics, we wonder \emph{whether our theoretical generalization guarantees and HNAS based on Theorem~\ref{th:connection} are also applicable to these non-gradient-based training-free metrics}. In particular, we firstly examine the correlation between the true generalization performances of all candidate architectures in NAS-Bench-201 and their generalization (Sec.~\ref{sec:generalization-nas}) using training-free metrics $\gM_{\fisher}$, $\gM_{\synflow}$ and $\gM_{\naswot}$. Table~\ref{tab:app:generalization-other-metrics} summarizes the results. Here, we use the same parameters applied in Sec. \ref{sec:exp:generalization} for Corollary~\ref{corol:generalization-non-realizable}. While $\gM_{\fisher}$ and $\gM_{\synflow}$ enjoy higher correlation to $\gM_{\tracenorm}$ than $\gM_{\naswot}$ in Appendix~\ref{sec:app:connection}, our generalization guarantees also performs better when using $\gM_{\fisher}$ and $\gM_{\synflow}$. We then apply our HNAS based on these training-free metrics in NAS-Bench-201 and the Table~\ref{tab:sota-nasbench201-other-metrics} summarizes the search results. Similarly, our HNAS based on $\gM_{\fisher}$ and $\gM_{\synflow}$ can also find better-performing architectures than HNAS ($\gM_{\naswot}$). Surprisingly, HNAS ($\gM_{\synflow}$) can even achieve competitive results when compared with HNAS using gradient-based training-free metrics. These results therefore indicate that our HNAS sometimes may also be able to improve over training-free NAS using non-gradient-based training-free metrics especially when these non-gradient-based training-free metrics contain certain gradient information.

\begin{table}[t]
\renewcommand\multirowsetup{\centering}
\caption{
Correlation between the test errors of all architectures in NAS-Bench-201 and our generalization guarantees in Sec.~\ref{sec:generalization-nas} using training-free metrics $\gM_{\knas}$, $\gM_{\fisher}$, $\gM_{\synflow}$ and $\gM_{\naswot}$ that are evaluated on various datasets.
Each correlation is reported with the mean and standard deviation using the metrics evaluated on CIFAR-10/100 and ImageNet-16-120.
}
\label{tab:app:generalization-other-metrics}
\centering
\resizebox{\columnwidth}{!}{
\begin{tabular}{lcccccccccc}
\toprule
\multirow{2}{*}{\textbf{Dataset}} &
\multicolumn{4}{c}{\textbf{Spearman}} &
\multicolumn{4}{c}{\textbf{Kendall's Tau}} \\
\cmidrule(l){2-5} \cmidrule(l){6-9}
& $\gM_{\knas}$ & 
$\gM_{\fisher}$ &
$\gM_{\synflow}$ &
$\gM_{\naswot}$ &
$\gM_{\knas}$ & 
$\gM_{\fisher}$ &
$\gM_{\synflow}$ &
$\gM_{\naswot}$ \\
\midrule
& \multicolumn{8}{c}{\textbf{Realizable scenario}} \\
C10 & 0.53$\pm$0.02 & 0.39$\pm$0.01 & 0.78$\pm$0.00 & 0.09$\pm$0.02 & 0.39$\pm$0.02 & 0.29$\pm$0.01 & 0.58$\pm$0.00 & 0.10$\pm$0.00 \\
C100 & 0.53$\pm$0.03 & 0.39$\pm$0.01 & 0.76$\pm$0.00 & 0.09$\pm$0.02 & 0.38$\pm$0.02 & 0.29$\pm$0.01 & 0.57$\pm$0.00 & 0.11$\pm$0.01 \\
IN-16 & 0.46$\pm$0.02 & 0.32$\pm$0.01 & 0.75$\pm$0.00 & 0.16$\pm$0.02 & 0.33$\pm$0.02 & 0.24$\pm$0.01 & 0.56$\pm$0.00 & 0.15$\pm$0.02 \\
\midrule
& \multicolumn{8}{c}{\textbf{Non-realizable scenario}} \\
C10 & 0.66$\pm$0.02 & 0.51$\pm$0.00 & 0.81$\pm$0.00 & 0.05$\pm$0.00 & 0.49$\pm$0.02 & 0.37$\pm$0.00 & 0.61$\pm$0.00 & 0.03$\pm$0.00 \\
C100 & 0.67$\pm$0.03 & 0.51$\pm$0.01 & 0.80$\pm$0.02 & 0.05$\pm$0.01 & 0.49$\pm$0.02 & 0.37$\pm$0.00 & 0.60$\pm$0.00 & 0.03$\pm$0.00 \\
IN-16 & 0.62$\pm$0.04 & 0.44$\pm$0.00 & 0.78$\pm$0.00 & 0.05$\pm$0.01 & 0.45$\pm$0.03 & 0.32$\pm$0.00 & 0.59$\pm$0.00 & 0.03$\pm$0.00 \\
\bottomrule
\end{tabular}
}
\end{table}

\begin{table}[t]
\caption{Comparison among HNAS using different training-free metrics in NAS-Bench-201. The performance of each HNAS variant is reported with the mean and standard deviation of five independent searches and the search costs are evaluated on a single Nvidia 1080Ti.}
\label{tab:sota-nasbench201-other-metrics}
\centering
\begin{tabular}{lccccccc}
\toprule
\multirow{2}{*}{\textbf{Algorithm}} & \multicolumn{3}{c}{\textbf{Test Accuracy (\%)}} &
\multirow{2}{*}{\textbf{Search Cost}} \\
\cmidrule(l){2-4} 
& C10 & C100 & IN-16 & (GPU Sec.) \\
\midrule
HNAS ($\gM_{\gradnorm}$) & 94.04$\pm$0.21 & 71.75$\pm$1.04 & 45.91$\pm$0.88 & 3010  \\
HNAS ($\gM_{\snip}$) & 93.94$\pm$0.02 & 71.49$\pm$0.11 & 46.07$\pm$0.14 & 2976 \\
HNAS ($\gM_{\grasp}$) & 94.13$\pm$0.13 & 72.59$\pm$0.82 & 46.24$\pm$0.38 & 3148  \\
HNAS ($\gM_{\tracenorm}$) & 94.07$\pm$0.10 & 72.30$\pm$0.70 & 45.93$\pm$0.37 & 3006 \\
\midrule
HNAS ($\gM_{\knas}$) & 94.19$\pm$0.06 & 72.94$\pm$0.52 & 46.31$\pm$0.38 & 3081 \\
\midrule
HNAS ($\gM_{\fisher}$) & 93.28$\pm$0.73 & 69.42$\pm$1.36 & 42.85$\pm$2.09 & 3309\\
HNAS ($\gM_{\synflow}$) & 94.13$\pm$0.00 & 72.50$\pm$0.00 & 45.47$\pm$0.00 & 3615 \\
HNAS ($\gM_{\naswot}$) & 92.10$\pm$0.62 & 66.81$\pm$0.32 & 39.26$\pm$0.72 & 2832 \\
\midrule
\textbf{Optimal} & 94.37 & 73.51 & 47.31 & - & \\
\bottomrule
\end{tabular}
\end{table}

\paragraph{Ablation Study on Optimization Process of HNAS.} In this section, we examine the evolution of the correlation between the test errors of candidate architectures in the NAS search space and their generalization guarantees in the non-realizable scenario with the increasing BO iterations in our HNAS framework. Figure~\ref{fig:app:evolve} illustrates the results in NAS-Bench-201 with CIFAR-10 dataset and training-free metric $\gM_{\tracenorm}$. Note that in every BO iteration of Figure \ref{fig:app:evolve}, the Spearman correlation we reported corresponds to the pair of hyperparameters $\mu$ and $\nu$ that achieves the best validation performance in the query history. As shown in Figure \ref{fig:app:evolve}, our HNAS framework, interestingly, is indeed selecting better-performing architectures by selecting hyperparameters $\mu$ and $\nu$ that can achieve higher Spearman correlation in the search space. These results therefore further justify the advantages of introducing BO algorithms with training-based performances into training-free NAS for better characterization.

\paragraph{Ablation Study on Training-Free Variant of HNAS.} According to \eqref{eq:nas-final} in our main paper, a completely training-free metric can be produced by simply specifying the values of $\mu$ and $\nu$ with prior knowledge. For example, by setting $\mu=0$, we can obtain the training-free metric $\kappa/\gM_{\tracenorm}$. However, obtaining prior knowledge regarding the best choice of $\mu$ and $\nu$ for NAS is non-trivial. Therefore, tuning $\mu$ and $\nu$ would be a better alternative to achieve more competitive search results in practice. To demonstrate this, we compare the performance of the architecture selected from training-free metric $\kappa/\gM_{\tracenorm}$ vs. our standard HNAS framework in Table \ref{tab:sota-nasbench201-hnas-free}. Notably, the results in Table \ref{tab:sota-nasbench201-hnas-free} demonstrate that tuning $\mu$ and $\nu$ based on training-based performances can indeed lead to improved search results and therefore will be a better alternative than pre-defining $\mu$ and $\nu$ for a completely training-free NAS, which further justifies the essence of combining training-free and training-based methods (as one of our major contributions) in HNAS.

\begin{table}[t]
\caption{Comparison between HNAS and its training-free variant in NAS-Bench-201.}
\label{tab:sota-nasbench201-hnas-free}
\centering
\begin{tabular}{lcccccc}
\toprule
\multirow{2}{*}{\textbf{Algorithm}} & \multicolumn{3}{c}{\textbf{Test Accuracy (\%)}} \\
\cmidrule(l){2-4} 
& C10 & C100 & IN-16 \\
\midrule
$\kappa/\gM_{\tracenorm}$ & 93.50 & 69.78 & 43.73 \\
HNAS ($\gM_{\tracenorm}$) & 94.10$\pm$0.16 & 72.48$\pm$0.95 & 46.30$\pm$0.17 \\
\midrule
\textbf{Optimal} & 94.37 & 73.51 & 47.31 & \\
\bottomrule
\end{tabular}
\end{table}

\begin{figure}[t]
\centering
\includegraphics[width=0.6\columnwidth]{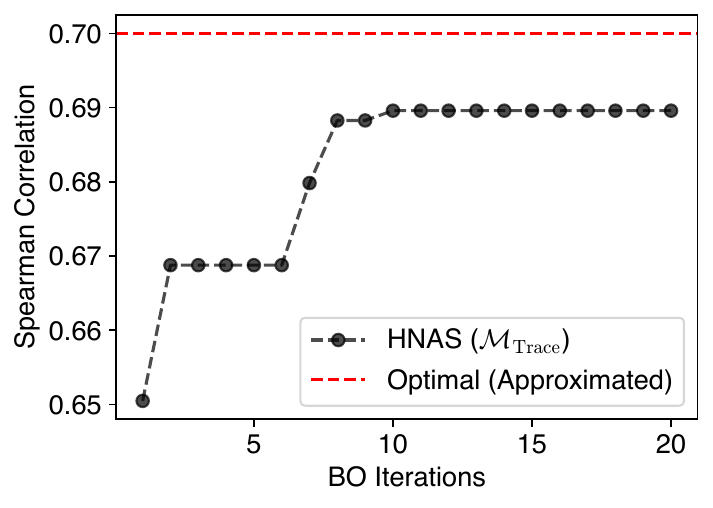}
\caption{Evolution of the correlation between the test errors (on CIFAR-10) of all architectures in NAS-Bench-201 and their generalization guarantees (using $\gM_{\tracenorm}$) in the non-realizable scenario with the BO iterations in our HNAS framework.}
\label{fig:app:evolve}
\end{figure}

\paragraph{Ablation Study on BO algorithm in HNAS.} To investigate the influence of different BO algorithms (i.e., different acquisition functions) on the optimization part of our HNAS, we compare the search results obtained from using different acquisition functions (i.e., expected improvement (EI) vs. upper confidence bound (UCB)) with HNAS($\gM_{\tracenorm}$) and HNAS($\gM_{\gradnorm}$) on NAS-Bench-201 in Table \ref{tab:sota-nasbench201-other-acquisition}. 
Since the default hyperparameters for different acquisition functions in \citep{bo-imp} have already been tuned for a variety of tasks, we directly make use of them in our experiments without any changes.
Interestingly, the results in Table \ref{tab:sota-nasbench201-other-acquisition} show that different acquisition functions (i.e., different BO algorithms) typically have limited influence on our HNAS framework. That is, our HNAS is shown to be relatively robust to the change of acquisition function in BO. 

\begin{table}[t]
\caption{Comparison among HNAS using different acquisition functions in NAS-Bench-201.~The performance of each HNAS variant is reported with the mean and standard deviation of five independent searches.}
\label{tab:sota-nasbench201-other-acquisition}
\centering
\begin{tabular}{lcccccc}
\toprule
\multirow{2}{*}{\textbf{Algorithm}} & \multicolumn{3}{c}{\textbf{Test Accuracy (\%)}} \\
\cmidrule(l){2-4} 
& C10 & C100 & IN-16 \\
\midrule
HNAS ($\gM_{\gradnorm}$) \, w/ EI & 94.04$\pm$0.21 & 71.75$\pm$1.04 & 45.91$\pm$0.88  \\
HNAS ($\gM_{\gradnorm}$) \, w/ UCB & 94.05$\pm$0.18 & 72.04$\pm$1.18 & 45.81$\pm$0.88  \\
\midrule
HNAS ($\gM_{\tracenorm}$) \  w/ EI & 94.07$\pm$0.10 & 72.30$\pm$0.70 & 45.93$\pm$0.37 \\
HNAS ($\gM_{\tracenorm}$) \  w/ UCB & 94.10$\pm$0.16 & 72.48$\pm$0.95 & 46.30$\pm$0.17 \\
\midrule
\textbf{Optimal} & 94.37 & 73.51 & 47.31 & \\
\bottomrule
\end{tabular}
\end{table}

\end{appendices}

\end{document}